\documentclass[11pt,letter]{article}
\pdfoutput=1
\usepackage{fullpage}

\usepackage{authblk}

\usepackage{graphicx} 
\usepackage{pgf}
\usepackage{subfigure}
\usepackage{enumitem}
\usepackage{footmisc}

\usepackage{complexity} 

\usepackage[round]{natbib}

\usepackage{algorithm}
\usepackage{algorithmic}

\usepackage[colorlinks, citecolor=blue, urlcolor=blue]{hyperref}

\usepackage{rfmacros}

\newcommand{\includePgfSized}[2][\textwidth]{\resizebox{#1}{!}{\input{#2}}}

\newcommand\gap{\operatorname{gap}}
\newcommand\dualappa{Dual APPA}
\newcommand\nsum[1]{\sum_{#1 = 1}^n}

\newcommand\itr[1]{{({#1})}}
\newcommand\oracle{\mathcal}
\newcommand\opt[1]{#1^\text{opt}}
\newcommand\otilde{\widetilde{O}}
\renewcommand\tlO{\otilde}

\newcommand\mnist{MNIST}        
\newcommand\cifar{CIFAR}        
\newcommand\protein{Protein}

\renewcommand\hat{\widehat}

\definecolor{royalpurple}{rgb}{0.3, 0.1, 0.8}

\title{Un-regularizing: approximate proximal point and faster stochastic algorithms for empirical risk minimization}

\date{}

\author[1]{Roy Frostig\thanks{This is an extended and updated version
of our conference paper that appeared
in
\textit{Proceedings of the 32$^{\text{nd}}$ International
Conference on Machine Learning}, Lille, France, 2015. JMLR:
W\&CP Volume 37.
\\
Email: \texttt{rf@cs.stanford.edu}, \texttt{rongge@microsoft.com},
\texttt{skakade@microsoft.com}, \texttt{sidford@mit.edu}.  }}
\author[2]{Rong Ge}
\author[2]{Sham M. Kakade}
\author[3]{Aaron Sidford}

\affil[1]{Stanford University}
\affil[2]{Microsoft Research, New England}
\affil[3]{MIT}

\begin{document}

\maketitle

\begin{abstract}
We develop a family of accelerated stochastic algorithms that
minimize sums of convex functions. Our algorithms improve upon the
fastest running time for empirical risk minimization (ERM), and in
particular linear least-squares regression, across a wide range of
problem settings.

To achieve this, we establish a framework based on the classical
\emph{proximal point algorithm}. Namely, we provide several
algorithms that reduce the minimization of a strongly convex
function to approximate minimizations of regularizations of the
function. Using these results, we accelerate recent fast stochastic
algorithms in a black-box fashion.

Empirically, we demonstrate that the resulting algorithms exhibit
notions of stability that are advantageous in practice.
Both in theory and in practice, the provided algorithms reap the
computational benefits of adding a large strongly convex
regularization term, without incurring a corresponding bias to the
original problem.
\end{abstract}

\nocite{nesterov1983acceleration,nesterov2004introductory}

\section{Introduction}

A general optimization problem central to machine learning is that of
\emph{empirical risk minimization} (ERM): finding a predictor or
regressor that minimizes a sum of loss functions defined by a data
sample. We focus in part on the problem of empirical risk minimization
of linear predictors: given a set of $n$ data points $a_i, \dots, a_n
\in \R^d$ and convex loss functions $\phi_i : \R \to \R$ for $i = 1,
\dots, n$, solve
\begin{align}
\label{eq:erm}
\min_{x \in \R^n} F(x),
\quad \text{ where } \quad
F(x) \defeq \nsum{i} \phi_i(a_i^{\T} x).
\end{align}
This problem underlies supervised learning (\eg the training of
logistic regressors when $\phi_i(z) = \log(1+e^{-zb_i})$, or their
regularized form when $\phi_i(z) = \log(1+e^{-zb_i}) + \tfrac \gamma
{2n} \norm{x}_2^2$ for a scalar $\gamma > 0$)
and captures the widely-studied problem of linear least-squares
regression when $\phi_i(z) = \tfrac 1 2 (z-b_i)^2$.

Over the past five years, problems such as \eqref{eq:erm} have
received increased attention, with a recent burst of activity in the
design of fast \emph{randomized} algorithms. Iterative methods that
randomly sample the $\phi_i$ have been shown to outperform standard first-order methods under mild assumptions
\citep{bottou08large, johnson13svrg, xiao2014proximal,
defazio2014saga, shalev2014accelerated}.

Despite the breadth of these recent results, their running time
guarantees when solving the ERM problem \eqref{eq:erm} are
sub-optimal in terms of their dependence on a natural notion of the
problem's \emph{condition number} (See Section~\ref{sec:setup}). This
dependence can, however, significantly impact their guarantees on
running time. High-dimensional problems encountered in practice are
often poorly conditioned. In large-scale machine learning
applications, the condition number of the ERM problem \eqref{eq:erm}
captures notions of data complexity arising from variable correlation
in high dimensions and is hence prone to be very large.

More specifically, among the recent randomized algorithms, each one
either:
\begin{enumerate}

\item Solves the ERM problem \eqref{eq:erm}, under an assumption of
strong convexity, with convergence that depends linearly on the
problem's condition number \citep{johnson13svrg, defazio2014saga}.
\item Solves only an explicitly \emph{regularized} ERM problem,
$\min_{x} \set{ F(x) + \lambda r(x) }$ where the regularizer $r$ is
a known 1-strongly convex function and $\lambda$ must be strictly
positive, even when $F$ is itself strongly convex.
One such result is due to \citet{shalev2014accelerated} and is the
first to achieve \emph{acceleration} for this problem, \ie dependence only on the
square root of the regularized problem's condition number, which
scales inversely with $\lambda$. Hence, taking small $\lambda$ to
solve the ERM problem (where $\lambda=0$ in effect) is not a viable
option.
\end{enumerate}

In this paper we show how to bridge this gap via black-box
reductions. Namely, we develop algorithms to solve the ERM problem
\eqref{eq:erm}~-- under a standard assumption of strong convexity~--
through repeated, approximate minimizations of the regularized ERM
problem $\min_x \set{F(x) + \lambda r(x)}$ for fairly large
$\lambda$. Instantiating our framework with known randomized
algorithms that solve the regularized ERM problem, we achieve
accelerated running time guarantees for solving the original ERM
problem.

The key to our reductions are approximate variants of the classical
\emph{proximal point algorithm} (PPA) \citep{rockafellar1976ppa,
parikh2014proximal}. We show how both PPA and the inner minimization
procedure can then be accelerated and our analysis gives precise
approximation requirements for either option. Furthermore, we show
further practical improvements when the inner minimizer operates by a
dual ascent method. In total, this provides at least three different
algorithms for achieving an improved accelerated running time for
solving the ERM problem \eqref{eq:erm} under the standard assumption
of strongly convex $F$ and smooth $\phi_i$. (Table~\ref{tab:rates}
summarizes our improvements in comparison to existing minimization
procedures.)

Perhaps the strongest and most general theoretical reduction we
provide in this paper is encompassed by the following theorem which we
prove in Section~\ref{sec:main_accel_ppa}.

\begin{thm}[Accelerated Approximate Proximal Point Algorithm]
\label{thm:intro-main}
Let $f : \R^n \rightarrow \R$ be a $\mu$-strongly convex function and
suppose that, for all $x_0 \in \R^n$, $c > 0$, $\lambda > 0$, we
can compute a point $x_{c}$ (possibly random) such that
\begin{align*}
\E f(x_{c}) - \min_{x} \bs{ f(x) + \frac{\lambda}{2} \norm{x - x_0}_2^2 }
\leq
\frac{1}{c} \left[
f(x_0) - \min_{x} \bs{ f(x) + \frac{\lambda}{2} \norm{x - x_0}_2^2 }
\right]
\end{align*}
in time ${\oracle T}_{c}$. Then given any $x_0$, $c > 0$, $\lambda \geq 2\mu$, we can compute $x_1$ such that
\begin{align*}
\E f(x_1) - \min_{x} f(x) \leq
\frac{1}{c} \left[
f(x_0) - \min_{x} f(x)
\right]
\end{align*}
in time $O\bp{ \oracle T_{4(\frac{2\lambda + \mu}{\mu})^{3/2}} \sqrt{\lceil \lambda / \mu \rceil} \log c}$.
\end{thm}
This theorem essentially states that we can use a linearly convergent
algorithm for minimizing $f(x) + \lambda \norm{x - x_0}_2^2$ in order
to minimize $f$, while incurring a multiplicative overhead of only
$O(\sqrt{\lceil \lambda / \mu \rceil} \polylog(\lambda /
\mu))$. Applying this theorem to previous state-of-the-art algorithms
improves both the running time for solving \eqref{eq:erm}, as well as
the following more general ERM problem:
\begin{align}
\min_{x \in \R^d} \nsum{i} \psi_i(x),
\quad \text{ where } \quad
\psi_i : \R^d \to \R .
\label{eq:erm_general}
\end{align}
Problem \eqref{eq:erm_general} is fundamental in the theory of convex
optimization and covers ERM problems for multiclass and structured
prediction.

There are a variety of additional extensions to the ERM problem to
which some of our analysis easily applies. For instance, we could work
in more general normed spaces, allow non-uniform smoothness of the
$\phi$, add an explicit regularizer,
etc. However, to simplify exposition and comparison to related work,
we focus on \eqref{eq:erm} and make clear the extensions to
\eqref{eq:erm_general} in Section~\ref{sec:main_accel_ppa}. These
cases capture the core of the arguments presented and illustrate the
generality of this approach.

Several of the algorithmic tools and analysis techniques in this paper
are similar in principle to (and sometimes appear indirectly in) work
scattered throughout the machine learning and optimization
literature~-- from classical treatments of error-tolerant PPA
\citep{rockafellar1976ppa, guler1992ppa} to the effective proximal
term used by Accelerated Proximal SDCA \citet{shalev2014accelerated}
in enabling its acceleration.

By analyzing these as separate tools, and by bookkeeping the error
requirements that they impose, we are able to assemble them into
algorithms with improved guarantees.
We believe that the presentation of Accelerated APPA
(Algorithm~\ref{alg:accel-appa}) arising from this view simplifies,
and clarifies in terms of broader convex optimization theory, the
``outer loop'' steps employed by Accelerated Proximal SDCA.
More generally, we hope that disentangling the relevant algorithmic
components into this general reduction framework will lead to further
applications both in theory and in practice.

\subsection{Formal setup}
\label{sec:setup}

We consider the ERM problem \eqref{eq:erm} in the following common setting:
\begin{assm}[Regularity]
\label{assm:regularity}
Each loss function $\phi_i$ is $L$-smooth, \ie for all $x,
y \in \R$,
\begin{align*}
\phi(y) &\leq \phi(x) + \phi'(x) (y - x) + \frac{L}{2} (y - x)^2,
\end{align*}
and the sum $F$ is $\mu$-strongly convex, \ie for all $x, y \in
\R^d$,
\begin{align*}
F(x) &\geq F(x) + \grad F(x)^\T (y - x) + \frac{\mu}{2} \norm{y - x}_2^2.
\end{align*}
\end{assm}
We let $R \defeq \max_i \norm{a_i}_2$ and let $A \in
\R^{n \by d}$ be the matrix whose $i$'th row is $a_i^\T$. We refer to
\begin{align*}
\kappa \defeq \ceil{ L R^2/\mu }
\end{align*}
as the \emph{condition number} of \eqref{eq:erm}.

Although many algorithms are designed for special cases of the ERM
objective $F$ where there is some known, exploitable structure to the
problem, our aim is to study the most general case subject to
Assumption~\ref{assm:regularity}. To standardize the comparison among
algorithms, we consider the following generic model of interaction with $F$:

\begin{assm}[Computational model]
\label{assm:oraclemodel}
For any $i \in [n]$ and $x \in \R^d$, we consider two primitive
operations:
\begin{itemize}
\item For $b \in \R$, compute the gradient of $x \mapsto
\phi_i(a_i^\T x - b)$.
\item For $b \in \R$, $c \in \R^d$, minimize $\phi_i(a_i^\T x) +
b \norm{x - c}_2^2$.
\end{itemize}
We refer to these operations, as well as to the evaluation of
$\phi_i(a_i^\T x)$, as single \emph{accesses} to $\phi_i$,
and assume that these operations can be performed in $O(d)$ time.
\end{assm}

\paragraph{Notation} Denote $[n] \defeq \set{1,\dots,n}$. Denote the
optimal value of a convex function by $\opt f = \min_x f(x)$, and,
when $f$ is clear from context, let $\opt{x}$ denote a minimizer.  A
point $x'$ is an \emph{$\ep$-approximate minimizer} of $f$ if $f(x') -
\opt f \leq \ep$. The Fenchel dual of a convex function $f : \R^k \to
\R$ is $f^* : \R^k \to \R$ defined by $f^*(y) = \sup_{x \in \R^k}
\set{\innprod{ y, x } - f(x)}$. We use $\otilde(\cdot)$ to hide
factors polylogarithmic in $n$, $L$, $\mu$, $\lambda$, and $R$, \ie $\otilde(f) =
O(f \polylog(n, L, \mu, \lambda, R))$.

\paragraph{Regularization and duality} Throughout the paper we let $F
: \R^d \rightarrow \R$ denote a $\mu$-strongly convex function. For
certain results presented, $F$ must in particular be the ERM problem
\eqref{eq:erm}, while other statements hold more generally. We
make it clear on a case-by-case basis when $F$ must have the ERM
structure as in \eqref{eq:erm}.

Beginning in Section~\ref{sec:main_results} and throughout the
remainder of the paper, we frequently consider the function
$f_{s,\lambda}(x)$, defined for all $x,s \in \R^d$ and $\lambda > 0$
by
\begin{align}
f_{s,\lambda}(x) &\defeq F(x) + \tfrac \lambda 2 \norm{x-s}_2^2
\label{eq:reg-primal}
\end{align}
In such context, we let $\opt{x}_{s,\lambda} \defeq \argmin_{x}
f_{s,\lambda}(x)$ and we call
\begin{align*}
\kappa_\lambda \defeq \ceil{ L R^2/\lambda }
\end{align*}
the \emph{regularized condition number}.

When $F$ is indeed the ERM objective \eqref{eq:erm}, certain
algorithms for minimizing $f_{s,\lambda}$ operate in the
\emph{regularized ERM dual}.
Namely, they proceed by decreasing the negative dual objective
$g_{s,\lambda} : \R^n \to \R$, given by
\begin{align}
g_{s,\lambda}(y) &\defeq G(y) + \frac 1 {2\lambda} \norm{A^\T y}_2^2 - s^\T A^\T y,
\label{eq:reg-dual}
\end{align}
where $G(y) \defeq \sum_{i=1}^n \phi_i^*(y_i)$.  Similar to the above, we
let $\opt{y}_{s,\lambda} \defeq \argmin_{y} g_{s,\lambda}(y)$.

To make corresponding primal progress, dual-based algorithms make use
of the \emph{dual-to-primal} mapping, given by
\begin{align}
\hat x_{s,\lambda}(y) &\defeq s - \tfrac 1 \lambda A^\T y,
\label{eq:d2p}
\end{align}
and the \emph{primal-to-dual} mapping, given entrywise by
\begin{align}
\bigsqbra{ \hat y(x) }_i &\defeq \evalAt{ \deldel {\phi_i(z)} z } {z=a_i^\T x}
\label{eq:p2d}
\end{align}
for $i = 1, \dots, n$. (See Appendix~\ref{sec:appendix-problems} for a derivation of these facts and further properties of the dual.)

\subsection{Running times and related work}
\label{sec:related}

\newcommand\thiswork{\noalign{\vskip-2pt}\hspace{-0.3em}\scriptsize{\textbf{\textcolor{royalpurple}{This work:}}}}
\newcommand\thisworkbig{\textbf{\textcolor{royalpurple}{This work}}}

\begin{table}[t]
\hspace{-5pt}
\mbox{
\begin{tabular}{|l|c|c|} \hline
\colsHeader{3}{|c|}{ Empirical risk minimization } \\ \hline
Algorithm        & Running time                                & Problem  \\ \hline \hline
GD               & $d n^2 \kappa \log(\ep_0/\ep)$              & $F$ \\
Accel.\ GD       & $d n^{3/2}\sqrt{\kappa} \log(\ep_0/\ep)$    & $F$ \\
SAG, SVRG        & $d n \kappa \log(\ep_0/\ep)$                & $F$ \\
SDCA             & $d n \kappa_\lambda \log(\ep_0/\ep)$        & $F + \lambda r$ \\
AP-SDCA          & $d n \sqrt{\kappa_\lambda} \log(\ep_0/\ep)$ & $F + \lambda r$ \\
APCG             & $d n \sqrt{\kappa_\lambda} \log(\ep_0/\ep)$ & $F + \lambda r$ \\
\thisworkbig     & $d n \sqrt{\kappa} \log(\ep_0/\ep)$         & $F$ \\
\hline
\end{tabular}
\begin{tabular}{|l|c|c|} \hline
\colsHeader{3}{|c|}{ Linear least-squares regression } \\ \hline
Algorithm        & Running time                       & Problem \\ \hline \hline
Naive mult.\     & $nd^2$                             & $\norm{Ax-b}^2_2$ \\
Fast mult.\      & $nd^{\omega-1}$                    & $\norm{Ax-b}^2_2$ \\
Row sampling     & $(nd + d^\omega) \log(\ep_0/\ep)$  & $\norm{Ax-b}^2_2$ \\
OSNAP   & $(nd + d^\omega) \log(\ep_0/\ep)$  & $\norm{Ax-b}^2_2$ \\
R.\ Kaczmarz  & $dn\kappa\log(\ep_0/\ep)$          & $Ax=b$ \\
Acc.\ coord.\  & $dn\sqrt{\kappa}\log(\ep_0/\ep)$   & $Ax=b$ \\
\thisworkbig     & $dn\sqrt{\kappa}\log(\ep_0/\ep)$   & $\norm{Ax-b}^2_2$ \\
\hline
\end{tabular}
}
\caption{Theoretical performance comparison on ERM and linear regression.
Running times hold in expectation for randomized algorithms.
In the ``problem'' column for ERM, $F$ marks algorithms that can
optimize the ERM objective \eqref{eq:erm}, while $F+\lambda r$ marks
those that only solve the explicitly regularized problem.
For linear regression, $Ax=b$ marks algorithms that only solve
consistent linear systems, whereas $\norm{Ax-b}^2_2$ marks those that
more generally minimize the squared loss. The constant $\omega$ denotes the exponent of the matrix multiplication running time (currently below
$2.373$ \citep{williams2012omega}).
See Section~\ref{sec:related} for more detail on these algorithms and their running times.}
\label{tab:rates}
\end{table}

In Table~\ref{tab:rates} we compare our results with the running time
of both classical and recent algorithms for solving the ERM
problem~\eqref{eq:erm} and linear least-squares regression. Here we
briefly explain these running times and related work.

\paragraph{Empirical risk minimization}
In the context of the ERM problem,
GD refers to canonical gradient descent on $F$,
Accel.\ GD is Nesterov's accelerated gradient decent
\citep{nesterov1983acceleration, nesterov2004introductory},
SVRG is the stochastic variance-reduced gradient of
\citet{johnson13svrg},
SAG is the stochastic average gradient of \citet{leroux2012sag} and
\citet{defazio2014saga},
SDCA is the stochastic dual coordinate ascent of \citet{shalev13stochastic},
AP-SDCA is the Accelerated Proximal SDCA of
\citet{shalev2014accelerated} and APCG is the accelerated coordinate
algorithm of \citet{lin2014accelerated}.
The latter three algorithms are more restrictive in that they only
solve the explicitly regularized problem $F+\lambda r$, even if $F$ is
itself strongly convex (such
algorithms run in time inversely proportional to $\lambda$).

The running times of the algorithms are presented based on the setting
considered in this paper, \ie under Assumptions~\ref{assm:regularity}
and \ref{assm:oraclemodel}. Many of the algorithms can be applied in
more general settings (\eg even if the function $F$ is not strongly
convex) and have different convergence guarantees in those cases. The
running times are characterized by four parameters: $d$ is the data
dimension, $n$ is the number of samples, $\kappa = \lceil
LR^2/\mu\rceil$ is the condition number (for $F+\lambda r$ minimizers, the
condition number $\kappa_\lambda = \lceil LR^2/\lambda \rceil$ is used)
and $\epsilon_0/\epsilon$ is the ratio between the initial and desired
accuracy. Running times are stated per $\tlO$-notation; factors that
depend polylogarithmically on $n$, $\kappa$, and $\kappa_\lambda$ are ignored.

\paragraph{Linear least-squares regression} For the linear least-squares
regression problem, there is greater variety in the algorithms that
apply. For comparison, Table~\ref{tab:rates} includes Moore-Penrose
pseudoinversion -- computed via naive matrix multiplication and
inversion routines, as well as by their asymptotically fastest
counterparts -- in order to compute a closed-form solution via the
standard normal equations. The table also lists algorithms based on
the randomized Kaczmarz method \citep{strohmer09kaczmarz,
needell2014kaczmarz} and their accelerated variant
\citep{lee13coordinate}, as well as algorithms based on subspace
embedding (OSNAP) or row sampling \citep{nelson2013osnap,
li2013sampling, cohen2015sampling}.
Some Kaczmarz-based methods only apply to the more restrictive
problem of solving a consistent system (finding $x$ satisfying $Ax = b$) rather
than minimize the squared loss $\|Ax-b\|^2_2$. The running times
depend on the same four parameters $n,d,\kappa, \epsilon_0/\epsilon$
as before, except for computing the closed-form pseudoinverse, which
for simplicity we consider ``exact,'' independent of initial and
target errors $\epsilon_0/\epsilon$.

\paragraph{Approximate proximal point} The key to our improved running
times is a suite of approximate proximal point algorithms that we
propose and analyze. We remark that notions of error-tolerance in the
typical proximal point algorithm~-- for both its plain and accelerated
variants~-- have been defined and studied in prior work
\citep{rockafellar1976ppa, guler1992ppa}. However, these mainly
consider the cumulative \emph{absolute} error of iterates produced by
inner minimizers, assuming that such a sequence is somehow
produced. Since essentially any procedure of interest begins at some
initial point~-- and has runtime that depends on the \emph{relative}
error ratio between its start and end~-- such a view does not yield
fully concrete algorithms, nor does it yield end-to-end runtime upper
bounds such as those presented in this paper.

\paragraph{Additional related work} There is an immense body of
literature on proximal point methods and alternating direction method
of multipliers (ADMM) that are relevant to the approach in this paper;
see \citet{boyd2011admm, parikh2014proximal} for modern surveys.
We also note that the independent work of \citet{lin15catalyst}
contains results similar to some of those in this paper.

\subsection{Main results}
\label{sec:main_results}

All formal results in this paper are obtained through a framework that
we develop for iteratively applying and accelerating various
minimization algorithms.
When instantiated with recently-developed fast minimizers
we obtain, under Assumptions~\ref{assm:regularity}
and~\ref{assm:oraclemodel}, algorithms guaranteed to solve the ERM
problem in time $\tlO(nd\sqrt{\kappa} \log(1/\ep))$.

Our framework stems from a critical insight of the classical
\emph{proximal point algorithm (PPA)} or \emph{proximal iteration}: to
minimize $F$ (or more generally, any convex function) it suffices to
iteratively minimize 
\begin{align*}
f_{s,\lambda}(x) &\defeq F(x) + \tfrac \lambda 2 \norm{x-s}_2^2
\end{align*}
for $\lambda > 0$ and proper choice of \emph{center} $s \in
\R^d$.
PPA iteratively applies the update
\begin{align*}
x^\itr{t+1} \gets \argmin_x f_{x^\itr{t},\lambda}(x)
\end{align*}
and converges to the minimizer of $F$. The
minimization in the update is known as the \emph{proximal operator}
\citep{parikh2014proximal}, and we refer to it in the sequel as the
\emph{inner} minimization problem.

In this paper we provide three distinct \emph{approximate} proximal point
algorithms, \ie algorithms that do not require full inner
minimization. Each enables the use of existing
fast algorithm as its inner minimizer, in turn yielding several ways
to obtain our improved ERM running time:
\begin{itemize}
\item In Section~\ref{sec:main_ppa} we develop a basic approximate
proximal point algorithm (APPA). The algorithm is essentially PPA
with a relaxed requirement of inner minimization by only a
\emph{fixed} multiplicative constant in each
iteration. Instantiating this algorithm with an accelerated,
regularized ERM solver~-- such as APCG \citep{lin2014accelerated}~--
as its inner minimizer yields the improved accelerated running time
for the ERM problem \eqref{eq:erm}.

\item In Section~\ref{sec:main_accel_ppa} we develop Accelerated
APPA. Instantiating this algorithm with SVRG \citep{johnson13svrg}
as its inner minimizer yields the improved accelerated running time
for both the ERM problem \eqref{eq:erm} as well as the general ERM
problem \eqref{eq:erm_general}.

\item In Section~\ref{sec:main_dual_ppa} we develop Dual APPA: an
algorithm whose approximate inner minimizers operate on the dual
$f_{s,\lambda}$, with warm starts between iterations.
Dual APPA enables several inner minimizers that are a priori
incompatible with APPA. Instantiating this algorithm with an
accelerate, regularized ERM solver~-- such as APCG
\citep{lin2014accelerated}~-- as its inner minimizer yields the
improved accelerated running time for the ERM problem
\eqref{eq:erm}.
\end{itemize}

Each of the three algorithms exhibits a slight advantage over the
others in different regimes. APPA has by far the simplest and most
straightforward analysis, and applies directly to any $\mu$-strongly
convex function $F$ (not only $F$ given by
\eqref{eq:erm}). Accelerated APPA is more complicated, but in many
regimes is a more efficient reduction than APPA; it too applies to any
$\mu$-strongly convex function $F$ and in turn proves
Theorem~\ref{thm:intro-main}.

Our third algorithm, Dual APPA, is the least general in terms of the
assumptions on which it relies. It is the only reduction we develop
that requires the ERM structure of $F$. However, this algorithm is a
natural choice in conjunction with inner minimizers that operate on a
popular dual objective. In Section~\ref{sec:implementation} we
demonstrate moreover that this algorithm has properties that make it
desirable in practice.

\subsection{Paper organization}

The remainder of this paper is organized as follows. In
Section~\ref{sec:main_ppa}, Section~\ref{sec:main_accel_ppa}, and
Section~\ref{sec:main_dual_ppa} we state and analyze the
approximate proximal point algorithms described above.
In Section~\ref{sec:implementation} we discuss practical concerns and
cover numerical experiments involving Dual APPA and related stochastic
algorithms. 
In Appendix~\ref{sec:appendix-lemmas} we prove general technical
lemmas used throughout the paper and in
Appendix~\ref{sec:appendix-problems} we provide a derivation of
regularized ERM duality and related technical lemmas.

\section{Approximate proximal point algorithm (APPA)}
\label{sec:main_ppa}

In this section we describe our approximate proximal point algorithm
(APPA). This algorithm is perhaps the simplest, both in its
description and in its analysis, in comparison to the others described
in this paper. This section also introduces technical machinery that
is used throughout the sequel.

We first present our formal abstraction of inner minimizers (Section~\ref{sub:appa:oracles}), then we
present our algorithm (Section~\ref{sub:appa:algorithm}), and finally we step through its analysis (Section~\ref{sub:appa:analysis}).

\subsection{Approximate primal oracles}
\label{sub:appa:oracles}

To design APPA, we first quantify the error that can be tolerated of an
inner minimizer, while accounting for the computational cost of
ensuring such error. The abstraction we use is the following notion of
inner approximation:

\begin{defn}
\label{defn:primal-oracle}
An algorithm $\oracle P$ is a \emph{primal $(c,\lambda)$-oracle} if,
given $x \in \R^d$, it outputs $\oracle P(x)$ that is a
$([f_{x,\lambda}(x) - \opt f_{x,\lambda}] / c)$-approximate minimizer
of $f_{x,\lambda}$ in time $\oracle T_{\oracle P}$.\footnote{When
the oracle is a randomized algorithm, we require that expected error is the same, \ie that the solution be $\ep$-approximate in expectation.\label{foot:random-oracle}}
\end{defn}

In other words, a primal oracle is an algorithm initialized at $x$
that reduces the error of $f_{x,\lambda}$ by a $1/c$ fraction, in time
that depends on $\lambda$, and $c$, and regularity properties of $F$.
Typical iterative first-order algorithms, such as those in
Table~\ref{tab:rates}, yield primal $(c,\lambda)$-oracles with
runtimes $\oracle T_{\oracle P}$ that scale inversely in $\lambda$ or
$\sqrt{\lambda}$, and logarithmically in $c$. For instance:

\begin{thm}[SVRG as a primal oracle]
\label{thm:svrg_primal_oracle}
SVRG \citep{johnson13svrg} is a primal $(c,\lambda)$-oracle with runtime complexity
$\oracle T_{\oracle P} = O(nd\min\set{\kappa,\kappa_\lambda} \log c)$
for both the ERM problem \eqref{eq:erm}
and the general ERM problem \eqref{eq:erm_general}.
\end{thm}

\begin{thm}[APCG as an accelerated primal oracle]
\label{thm:accel_sdca_primal_oracle}
Using APCG \citep{lin2014accelerated} we can obtain a primal
$(c,\lambda)$-oracle with runtime complexity
$\oracle T_{\oracle P} = \otilde(nd \sqrt{\kappa_\lambda} \log c)$ for
the ERM problem~\eqref{eq:erm}.\footnote{AP-SDCA could likely also
serve as a primal oracle with the same guarantees. However, the
results in \cite{shalev2014accelerated} are stated assuming initial
primal and dual variables are zero. It is not directly clear how one
can provide a generic relative decrease in error from this specific
initial primal-dual pair.}
\end{thm}

\begin{proof}
Corollary~\ref{cor:init_dual_error} implies that, given a primal
point $x$, we can obtain, in $O(nd)$ time, a corresponding dual point
$y$ such that the duality gap $f_{x,\lambda}(x) + g_{x,\lambda}(y)$
(and thus the dual error) is at most $O(\poly(\kappa_\lambda))$
times the primal error.
Lemma~\ref{lem:dual-error-bounds-primal} implies that decreasing the
dual error by a factor $O(\poly(\kappa_\lambda) c)$ decreases
the induced primal error by $c$. Therefore, applying APCG to the dual
and performing the primal and dual mappings yield the
theorem.
\end{proof}

\subsection{Algorithm}
\label{sub:appa:algorithm}

Our Approximate Proximal Point Algorithm (APPA) is given by the following Algorithm~\ref{alg:primal-appa}.

\begin{algorithm}[H]
\begin{algorithmic}
\INPUT $x^\itr{0} \in \R^d$, $\lambda > 0$
\INPUT primal $(\frac{2(\lambda + \mu)}{\mu},\lambda)$-oracle $\oracle P$
\FOR{ $t = 1, \ldots, T$ }
\STATE $x^\itr{t} \gets \oracle P(x^\itr{t-1})$
\ENDFOR
\OUTPUT $x^\itr{T}$
\end{algorithmic}
\caption{Approximate PPA (APPA)}
\label{alg:primal-appa}
\end{algorithm}

The central goal of this section is to prove the following lemma, which guarantees a geometric convergence rate for the iterates produced in this manner

\begin{lem}[Contraction in APPA]
\label{lem:primal-appa-contract}
For any $c' \in (0,1)$, $x \in \R^d$, and possibly randomized primal
$(\frac {\lambda+\mu} {c'\mu}, \lambda)$-oracle $\oracle P$ (possibly randomized) we have
\begin{align}
\E[F(\oracle P(x))] - \opt F \leq \frac {\lambda + c'\mu} {\lambda+\mu} \bigpar{ F(x) - \opt F }.
\label{eq:primal-appa-contract}
\end{align}
\end{lem}

This lemma immediately implies the following
running-time bounds for APPA.

\begin{thm}[Un-regularizing in APPA]
\label{thm:unreg-appa}
Given a primal $(\frac{2(\mu+\lambda)}{\mu},\lambda)$-oracle
$\oracle P$,
Algorithm~\ref{alg:primal-appa} minimizes the general ERM problem
\eqref{eq:erm_general} to within accuracy $\ep$ in time $O({\oracle
T}_{\oracle P} \lceil \lambda / \mu \rceil \log(\epsilon_0 /
\epsilon))$.\footnote{When the oracle is a randomized algorithm, the
expected accuracy is at most $\ep$.
\label{foot:random-oracle-contractions}}
\end{thm}

Combining Theorem~\ref{thm:unreg-appa} and
Theorem~\ref{thm:accel_sdca_primal_oracle} immediately yields our
desired running time for solving \eqref{eq:erm}.

\begin{cor}
Instantiating Algorithm~\ref{alg:primal-appa} with the Theorem~\ref{thm:accel_sdca_primal_oracle} as the  primal oracle and taking $\lambda=\mu$
yields the running time of $\otilde( nd \sqrt{\kappa} \log(\ep_0/\ep) )$
for solving \eqref{eq:erm}.
\end{cor}

\subsection{Analysis}
\label{sub:appa:analysis}

This section gives a proof of Lemma~\ref{lem:primal-appa-contract}.
Throughout, no assumption is made on $F$ aside from $\mu$-strong
convexity. Namely, we need not have $F$ be smooth or at all
differentiable.

First, we consider the effect of an exact inner minimizer. Namely, we prove
the following lemma relating the minimum of the inner problem
$f_{s,\lambda}$ to $\opt F$.

\begin{lem}[Relationship between minima]
\label{lem:proximal_point_progress}
For all $s \in \R^d$ and $\lambda \geq 0$
\begin{align*}
\opt{f}_{s,\lambda} - \opt{F}
&\leq
\frac{\lambda}{\mu + \lambda} \bigpar{ F(s) - \opt{F} }.
\end{align*}
\end{lem}

\begin{proof}
Let $\opt{x} = \argmin_{x} F(x)$ and for all $\alpha \in [0, 1]$ let
$x_\alpha = (1 - \alpha)s + \alpha \opt{x}$. The $\mu$-strong
convexity of $F$ implies that, for all $\alpha \in [0, 1]$,
\begin{align*}
F(x_\alpha) &\leq (1 - \alpha) F(s) + \alpha F(\opt{x}) -
\frac{\alpha (1 - \alpha) \mu}{2} \norm{s - \opt{x}}_2^2.
\end{align*}
Consequently, by the definition of $\opt{f}_{s,\lambda}$,
\[
\opt{f}_{s,\lambda}
\leq
F(x_\alpha) + \frac \lambda 2 \norm{x_\alpha - s}_2^2
\leq
(1 - \alpha) F(s) + \alpha F(\opt{x})
- \frac{\alpha (1 - \alpha) \mu}{2} \norm{s - \opt{x}}_2^2
+ \frac{\lambda \alpha^2}{2} \norm{s - \opt{x}}_2^2
\]
Choosing $\alpha = \frac{\mu}{\mu + \lambda}$ yields the result.
\end{proof}

This immediately implies contraction for the exact PPA, as it implies
that in every iteration of PPA the error in $F$ decreases by a
multiplicative $\lambda / (\lambda + \mu)$. Using this we prove
Lemma~\ref{lem:primal-appa-contract}.

\begin{proof}[Proof of Lemma~\ref{lem:primal-appa-contract}]
Let $x' = P(x)$. By definition of primal oracle $P$ we have
\[
f_{x,\lambda}(x') - \opt{f}_{x,\lambda}
\leq
\frac {c'\mu} {\lambda+\mu} \bigpar{ f_{x,\lambda}(x) - \opt{f}_{x,\lambda} }.
\]
Combining this and Lemma~\ref{lem:proximal_point_progress} we have
\[
f_{x,\lambda}(x') - \opt{F}
\leq
\frac {c'\mu} {\lambda+\mu} \bigpar{ f_{x,\lambda}(x) - \opt{f}_{x,\lambda} }
+ \frac{\lambda}{\mu + \lambda} \bigpar{ F(x) - \opt{F} }
\]
Using that clearly for all $z$ we have $F(z) \leq f_{x,\lambda}(z)$ we see that  $F(x') \leq f_{x,\lambda}(x')$ and $\opt F \leq \opt{f}_{x,\lambda}$. Combining with the fact that $f_{x,\lambda}(x) = F(x)$ yields the result.
\end{proof}

\section{Accelerated APPA}
\label{sec:main_accel_ppa}

In this section we show how generically accelerate the APPA algorithm of
Section~\ref{sec:main_ppa}. Accelerated APPA
(Algorithm~\ref{alg:accel-appa}) uses inner minimizers more
efficiently, but requires a smaller minimization factor when compared
to APPA. The algorithm and its analysis immediately prove
Theorem~\ref{thm:intro-main} and in turn yield another means by which
we achieve the accelerated running time guarantees for solving
\eqref{eq:erm}.

We first present the algorithm and state its running time guarantees
(Section~\ref{sub:accel:algorithm}), then prove the guarantees as part
of analysis (Section~\ref{main:accel:analysis}).

\subsection{Algorithm}
\label{sub:accel:algorithm}

Our accelerated APPA algorithm is given by Algorithm~\ref{alg:accel-appa}. In every iteration it still makes a single call to a primal oracle, but rather than requiring a fixed constant minimization the minimization factor depends polynomial on the ratio of $\lambda$ and $\mu$.

\begin{algorithm}[H]
\begin{algorithmic}
\INPUT $x^\itr{0} \in \R^d$, $\mu > 0$, $\lambda > 2\mu$
\INPUT primal $(4\rho^{3/2},\lambda)$-oracle $\oracle P$, where $\rho = \frac {\mu+2\lambda} \mu$
\STATE Define $\zeta = \frac 2 \mu + \frac 1 \lambda$
\STATE $v^\itr{0} \gets x^\itr{0}$
\FOR{ $t = 0, \ldots, T-1$ }
\STATE $y^\itr{t} \gets \left(\frac 1 {1+\rho^{-1/2}}\right) x^\itr{t} + \left(\frac {\rho^{-1/2}} {1+\rho^{-1/2}}\right) v^\itr{t}$
\STATE $x^\itr{t+1} \gets \oracle P(y^\itr{t})$
\STATE $g^\itr{t} \gets \lambda( y^\itr{t} - x^\itr{t+1} )$
\STATE $v^\itr{t+1} \gets (1 - \rho^{-1/2}) v^\itr{t} + \rho^{-1/2} \bigsqbra{ y^\itr{t} - \zeta g^\itr{t} }$
\ENDFOR
\OUTPUT $x^\itr{T}$
\end{algorithmic}
\caption{Accelerated APPA}
\label{alg:accel-appa}
\end{algorithm}

The central goal is to prove the following theorem regarding the running time of APPA.

\begin{thm}[Un-regularizing in Accelerated APPA]
\label{thm:unreg-appa-accel}
Given a primal $( 4 (\frac{2\lambda+\mu}{\mu})^{3/2}, \lambda )$-oracle
$\oracle P$ for $\lambda \geq 2\mu$,
Algorithm~\ref{alg:accel-appa} minimizes the general ERM problem
\eqref{eq:erm_general} to within accuracy $\ep$ in time
$
O(
\oracle T_{\oracle P} 
\sqrt{\lceil \lambda / \mu \rceil} \log(\ep_0/\ep)
)
$.
\end{thm}

This theorem is essentially a restatement of Theorem~\ref{thm:intro-main} and by instantiating it with Theorem~\ref{thm:svrg_primal_oracle} we obtain the following.

\begin{cor}
Instantiating Theorem~\ref{thm:unreg-appa-accel} with SVRG
\citep{johnson13svrg} as the primal oracle and taking $\lambda = 2\mu + LR^2$
yields the running time bound $\tlO(nd\sqrt{\kappa}\log(\ep_0/\ep))$
for the general ERM problem \eqref{eq:erm_general}.
\end{cor}

\subsection{Analysis}
\label{main:accel:analysis}

Here we establish the convergence rate of
Algorithm~\ref{alg:accel-appa}, Accelerated APPA, and prove
Theorem~\ref{thm:unreg-appa-accel}. Note that as in Section~\ref{sec:main_ppa} the results in this section use nothing about the structure of $F$ other than strong
convexity and thus they apply to the general ERM problem
\eqref{eq:erm_general}.

We remark that aspects of the proofs in this section bear resemblance to the analysis in \citet{shalev2014accelerated}, which achieves similar results in a more specialized setting.

Our proof is split into the following parts.

\begin{itemize}
\item In Lemma~\ref{lem:outerlowerbound} we show that applying a primal oracle to the inner minimization problem gives us a quadratic lower bound on $F(x)$.

\item In Lemma~\ref{lem:accel_outer} we use this lower bound to construct a series of lower bounds for the main objective function $f$, and accelerate the APPA algorithm, comprising the bulk of the analysis.

\item In Lemma~\ref{lem:accel_init_error} we show that the requirements of Lemma~\ref{lem:accel_outer} can be met by using a primal oracle that decreases the error by a constant factor.

\item In Lemma~\ref{lem:outeraccel_init} we analyze the initial error requirements of Lemma~\ref{lem:accel_outer}.
\end{itemize}

The proof of Theorem~\ref{thm:unreg-appa-accel} follows immediately from these lemmas.

\begin{lem}
\label{lem:outerlowerbound}
For $x_0 \in \R^n$ and $\epsilon > 0$ suppose that $x^+$ is an $\epsilon$-approximate solution to $f_{x_0,\lambda}$. Then for $\mu' \defeq \mu/2$, $g \defeq \lambda (x_0-x^+)$, and all $x\in \R^n$ we have
$$
F(x) \ge F(x^+) - \frac{1}{2\mu'} \|g\|^2 + \frac{\mu'}{2}\left\|x - \left(x_0 - \left(\frac{1}{\mu'}+\frac{1}{\lambda}\right)g\right)\right\|_2^2 - \frac{\lambda+2\mu'}{\mu'}\epsilon.
$$
\end{lem}

Note that as $\mu' = \mu/2$ we are only losing a factor of 2 in the strong convexity parameter for our lower bound. This allows us to account for errors without sacrificing in our ultimate asymptotic convergence rates.

\begin{proof}
Since $F$ is $\mu$-strongly convex clearly $f_{x_0,\lambda}$ is $\mu+\lambda$ strongly convex, by Lemma~\ref{lem:smooth-sc-bounds}
\begin{equation}
\label{eq:outerlowerbound:1}
f_{x_0,\lambda}(x) - f_{x_0,\lambda}(\opt{x}_{x_0,\lambda}) \ge \frac{\mu+\lambda}{2} \|x-{\opt x}\|_2^2.
\end{equation}
By Cauchy-Schwartz and Young's Inequality we know that
\begin{align*}
\frac{\lambda+\mu'}{2}\|x-x^+\|_2^2 \le & \frac{\lambda+\mu'}{2} \left(\|x-\opt{x}_{x_0,\lambda}\|_2^2 + \|\opt{x}_{x_0,\lambda}-x^+\|_2^2\right) + \frac{\mu'}{2}\|x-\opt{x}_{x_0,\lambda}\|_2^2+ \frac{(\lambda+\mu')^2}{2\mu'}\|\opt{x}_{x_0,\lambda}-x^+\|_2^2,
\end{align*}
which implies
$$
\frac{\mu+\lambda}{2} \|x-\opt{x}_{x_0,\lambda}\|_2^2 \ge \frac{\lambda+\mu'}{2}\|x-x^+\|_2^2 - \frac{\lambda+\mu'}{\mu'}\cdot \frac{\lambda+\mu}{2} \|\opt{x}_{x_0,\lambda}-x^+\|_2^2.
$$

On the other hand, since $f_{x_0,\lambda}(x^+) \le f_{x_0,\lambda}(\opt{x}_{x_0,\lambda})+\epsilon$ by assumption we have $\frac{\lambda+\mu}{2}\|x^+-{\opt x}\|_2^2 \le \epsilon$ and therefore
\begin{align*}
f_{x_0,\lambda}(x)-f_{x_0,\lambda}(x^+) & \ge f_{x_0,\lambda}(x)-f_{x_0,\lambda}(\opt{x}_{x_0,\lambda}) -\epsilon
\ge \frac{\mu+\lambda}{2} \|x-\opt{x}_{x_0,\lambda}\|_2^2 -\epsilon
\\ & \ge \frac{\lambda+\mu'}{2}\|x-x^+\|_2^2 - \frac{\lambda+\mu'}{\mu'}\cdot \frac{\lambda+\mu}{2} \|\opt{x}_{x_0,\lambda}-x^+\|_2^2-\epsilon
\\& \ge \frac{\lambda+\mu'}{2}\|x-x^+\|_2^2 - \frac{\lambda+2\mu'}{\mu'}\epsilon.
\end{align*}

Now since
\begin{align*}
\|x-x^+\|_2^2&=\|x-x_0+\frac{1}{\lambda}g\|_2^2
=\|x-x_0\|^2+\frac{2}{\lambda}\left<g,x-x_0\right>+\frac{1}{\lambda^2}\|g\|_2^2,
\end{align*}
and using the fact that $f_{x_0,\lambda}(x) = F(x) + \frac{\lambda}{2}\| x-x_0\|_2^2$, we have
\begin{align*}
F(x) \ge& F(x^+)+\left[\frac{1}{\lambda}+\frac{\mu'}{2\lambda^2}\right]\|g\|_2^2+\left(1+\frac{\mu'}{\lambda}\right)\left<g,x-x_0\right>+\frac{\mu'}{2}\|x-x_0\|^2-\frac{\lambda+2\mu'}{\mu'}\epsilon.
\end{align*}
The right hand side of the above equation is a quadratic function. Looking at its gradient with respect to $x$ we see that it obtains its minimum when $x = x_0-(\frac{1}{\mu'}+\frac{1}{\lambda})g$ and has a minimum value of $F(x^+)-\frac{1}{2\mu'}\|g\|_2^2 - \frac{\lambda+2\mu'}{\mu'}\epsilon$.
\end{proof}

\begin{lem}
\label{lem:accel_outer}
Suppose that in each iteration $t$ we have $\psi_t \defeq \opt \psi_t + \frac{\mu'}{2}\|x-v^{\itr{t}}\|_2^2$ such that $F(x)\ge \psi_t(x)$ for all $x$. Let $\rho \defeq \frac{\mu'+\lambda}{\mu'}$ for $\lambda \geq 3\mu'$, and let
\begin{itemize}
\item $y^{\itr{t}} \defeq \bp{\frac{1}{1+\rho^{-1/2}}} x^{\itr{t}} + \bp{\frac{\rho^{-1/2}}{1+\rho^{-1/2}}} v^{\itr{t}}$,
\item $\E[f_{y^{\itr{t}},\lambda}(x^{\itr{t+1}})] - \opt f_{y^{\itr{t}},\lambda} \le \frac{\rho^{-3/2}}{4} (F(x^{\itr{t}})-\opt \psi_t)$,
\item $g^{\itr{t}} \defeq \lambda (y^{\itr{t}}-x^{\itr{t+1}})$,
\item $v^{\itr{t+1}} \defeq (1-\rho^{-1/2}) v^{\itr{t}} + \rho^{-1/2}\left[y^{\itr{t}} - \left(\frac{1}{\mu'}+\frac{1}{\lambda}\right)g^{\itr{t}}\right]$.
\end{itemize}
We have
$$
\E[F(x^{\itr{t}}) - \opt \psi_t] \le \left(1-\frac{\rho^{-1/2}}{2}\right)^t(F(x_0)-\opt \psi_0).
$$
\end{lem}

\begin{proof}
Regardless of how $y^{\itr{t}}$ is chosen we know by Lemma~\ref{lem:outerlowerbound} that for $\gamma = 1+\frac{\mu'}{\lambda}$ and all $x\in \R^n$
\begin{equation}
F(x) \ge F(x^{\itr{t+1}}) - \frac{1}{2\mu'}\|g^{\itr{t}}\|_2^2 + \frac{\mu'}{2}\left\|x - \left(y^{\itr{t}} - \frac{\gamma}{\mu'}g^{\itr{t}}\right)\right\|_2^2 - \frac{\lambda+2\mu'}{\mu'}\left(f_{y^{\itr{t}},\lambda}(x^{\itr{t+1}}) - \opt f_{y^{\itr{t}},\lambda}\right).\label{eqn:outerlower}
\end{equation}

Thus, for $\beta = 1 - \rho^{-1/2}$ we can let
\begin{align*}
\psi_{t+1}(x) & \defeq \beta \psi_t(x) + (1-\beta)\left[F(x^{\itr{t+1}}) - \frac{1}{2\mu'}\|g^{\itr{t}}\|_2^2 + \frac{\mu'}{2}\|x - \left(y^{\itr{t}} - \frac{\gamma}{\mu'}g^{\itr{t}}\right)\|_2^2\right.
\\ & \quad - \left.\frac{\lambda+2\mu'}{\mu'}(f_{y^{\itr{t}},\lambda}(x^{\itr{t+1}}) - \opt f_{y^{\itr{t}},\lambda})\right] \\
& =\beta \left[\opt \psi_t +\frac{\mu'}{2}\|x-v^{\itr{t}}\|_2^2\right]+ (1-\beta)\left[F(x^{\itr{t+1}}) - \frac{1}{2\mu'}\|g^{\itr{t}}\|_2^2 + \frac{\mu'}{2}\|x - \left(y^{\itr{t}} - \frac{\gamma}{\mu'}g^{\itr{t}}\right)\|_2^2 \right.
\\ &\quad - \left.\frac{\lambda+2\mu'}{\mu'}(f_{y^{\itr{t}},\lambda}(x^{\itr{t+1}}) - \opt f_{y^{\itr{t}},\lambda})\right]\\
&= \opt \psi_{t+1} + \frac{\mu'}{2}\|x-v^{\itr{t+1}}\|_2^2.
\end{align*}
where in the last line we used Lemma~\ref{lem:mergequadratic}. Again, by Lemma~\ref{lem:mergequadratic} we know that
\begin{align*}
\opt \psi_{t+1} & = \beta \psi_t + (1-\beta)\left(F(x^{\itr{t+1}})-\frac{1}{2\mu'}\|g^{\itr{t}}\|_2^2 - \frac{\lambda+2\mu'}{\mu'}(f_{y^{\itr{t}},\lambda}(x^{\itr{t+1}}) - \opt f_{y^{\itr{t}},\lambda})\right) \\ & \quad + \beta(1-\beta)\frac{\mu'}{2}\|v^{\itr{t}}-\left(y^{\itr{t}}-\frac{\gamma}{\mu'}g^{\itr{t}}\right)\|_2^2 \\
& \ge \beta \psi_t + (1-\beta) F(x^{\itr{t+1}}) - \frac{(1-\beta)^2}{2\mu'}\|g^{\itr{t}}\|_2^2 + \beta(1-\beta)\gamma \left<g^{\itr{t}},v^{\itr{t}}-y^{\itr{t}}\right> \\ &\quad - \frac{(1-\beta)(\lambda+2\mu')}{\mu'}(f_{y^{\itr{t}},\lambda}(x^{\itr{t+1}}) - \opt f_{y^{\itr{t}},\lambda}).
\end{align*}
In the second step we used the following fact:
$$
-\frac{1-\beta}{2\mu'} + \beta(1-\beta)\frac{\mu'}{2}\cdot \frac{\gamma^2}{\mu'} = \frac{1-\beta}{2\mu'}(-1+\beta \gamma^2) \ge -\frac{(1-\beta)^2}{2\mu'}.
$$
Furthermore, expanding the term $\frac{\mu}{2}\|(x-y^{\itr{t}})+\frac{\gamma}{\mu}g^{\itr{t}}\|_2^2$ and instantiating $x$ with $x^{\itr{t}}$ in (\ref{eqn:outerlower}) yields
$$
F(x^{\itr{t+1}}) \le F(x^{\itr{t}}) - \frac{1}{\lambda}\|g^{\itr{t}}\|_2^2 + \gamma \left<g^{\itr{t}},y^{\itr{t}}-x^{\itr{t}}\right>+\frac{\lambda+2\mu'}{\mu'}(f_{y^{\itr{t}},\lambda}(x^{\itr{t+1}}) - \opt f_{y^{\itr{t}},\lambda}).
$$

Consequently we know
\begin{align*}
F(x^{\itr{t+1}}) - \opt\psi_{t+1} &\le \beta[f(x^{\itr{t}})-\opt \psi_t] + \left[\frac{(1-\beta)^2}{2\mu'}-\frac{\beta}{\lambda}\right]\|g^{\itr{t}}\|_2^2 +\gamma \beta \left<g^{\itr{t}}, y^{\itr{t}}-x^{\itr{t}}-(1-\beta)(v^{\itr{t}}-y^{\itr{t}})\right> \\ &\quad + \frac{(\lambda+2\mu')}{\mu'}(f_{y^{\itr{t}},\lambda}(x^{\itr{t+1}}) - \opt f_{y^{\itr{t}},\lambda})
\end{align*}
Note that we have chosen $y^{\itr{t}}$ so that the inner product term equals $0$, and we choose $\beta = 1-\rho^{-1/2} \geq \frac{1}{2}$ which ensures
$$
\frac{(1-\beta)^2}{2\mu'} - \frac{\beta}{\lambda}
\le \frac{1}{2(\mu'+\lambda)} - \frac{1}{2\lambda}
\le 0.
$$
Also, by assumption we know $\E[f_{y^{\itr{t}},\lambda}(x^{\itr{t+1}}) - \opt f_{y^{\itr{t}},\lambda}] \le \frac{\rho^{-3/2}}{4} (f(x^{\itr{t}})-\opt \psi_t)$, which implies
$$
\E[F(x^{\itr{t+1}}) - \opt\psi_{t+1}] \le \left(\beta+\frac{(\lambda+2\mu')}{\mu'}\cdot \frac{\rho^{-3/2}}{4} \right) (F(x^{\itr{t}})-\opt \psi_t) \le (1-\rho^{-1/2}/2) (F(x^{\itr{t}})-\opt \psi_t).
$$
In the final step we are using the fact that $\frac{\lambda + 2\mu'}{\mu'} \leq 2\rho$ and $\rho \geq 1$.
\end{proof}

\begin{lem}\label{lem:accel_init_error}
Under the setting of Lemma~\ref{lem:accel_outer}, we have $f_{y^{\itr{t}},\lambda}(x^{\itr{t}}) - \opt f_{y^{\itr{t}},\lambda} \le F(x^{\itr{t}}) - \opt \psi_t$. In particular, in order to achieve $\E[f_{y^{\itr{t}},\lambda}(x^{\itr{t+1}})] \le \frac{\rho^{-3/2}}{8}(F(x^{\itr{t}}) - \opt \psi_t)$ we only need an oracle that shrinks the function error by a factor of $\frac{\rho^{-3/2}}{8}$ (in expectation).
\end{lem}

\begin{proof}
We know
$$
f_{y^{\itr{t}},\lambda}(x^{\itr{t}}) - f(x^{\itr{t}}) = \frac{\lambda}{2}\|x^{\itr{t}}-y^{\itr{t}}\|_2^2 = \frac{\lambda}{2}\cdot \frac{\rho^{-1}}{(1+\rho^{-1/2})^2} \|x^{\itr{t}}-v^{\itr{t}}\|_2^2.
$$
We will try to show the lower bound $\opt f_{y^{\itr{t}},\lambda}$ is larger than $\opt \psi_t$ by the same amount. This is because for all $x$ we have
$$
f_{y^{\itr{t}},\lambda}(x) = F(x) + \frac{\lambda}{2}\|x-y^{\itr{t}}\|_2^2 \ge \opt \psi_t + \frac{\mu'}{2}\|x-v^{\itr{t}}\|_2^2 + \frac{\lambda}{2}\|x-y^{\itr{t}}\|_2^2.
$$
The right hand side is a quadratic function, whose optimal point is at $x = \frac{\mu'v^{\itr{t}}+\lambda y^{\itr{t}}}{\mu'+\lambda}$ and whose optimal value is equal to
$$
\opt \psi_t+\frac{\lambda}{2} \left(\frac{\mu'}{\mu' + \lambda}\right)^2\|v^{\itr{t}}-y^{\itr{t}}\|_2^2 + \frac{\mu'}{2}\left(\frac{\lambda}{\mu + \lambda}\right)^2 \|v^{\itr{t}}-y^{\itr{t}}\|_2^2 = \opt \psi_t +\frac{\mu'\lambda}{2(\mu'+\lambda)} \cdot \frac{1}{(1+\rho^{-1/2})^2} \|x^{\itr{t}}-v^{\itr{t}}\|_2^2.
$$
By definition of $\rho^{-1}$, we know $\frac{\mu'\lambda}{2(\mu'+\lambda)} \cdot \frac{1}{(1+\rho^{-1/2})^2} \|x^{\itr{t}}-v^{\itr{t}}\|_2^2$ is exactly equal to $\frac{\lambda}{2}\cdot \frac{\rho^{-1}}{(1+\rho^{-1/2})^2} \|x^{\itr{t}}-v^{\itr{t}}\|_2^2$, therefore $f_{y^{\itr{t}},\lambda}(x^{\itr{t}}) - \opt f_{y^{\itr{t}},\lambda} \le F(x^{\itr{t}}) - \opt \psi_t$.
\end{proof}

\paragraph{Remark} In the next lemma we show that moving to the
regularized problem has the same effect on the primal function value
and the lower bound. This is a result of the choice of $\beta$ in the
proof of Lemma~\ref{lem:accel_outer}. However, this does not mean that the
choice of $\beta$ is very fragile. We can choose any $\beta'$ that is
between the current $\beta$ and $1$; the effect on this lemma will
be that the increase in primal function becomes smaller than the increase in
the lower bound (so the lemma continues to hold).

\begin{lem}
\label{lem:outeraccel_init}
Let $\opt \psi_0 = F(x^{\itr{0}}) - \frac{\lambda+2\mu'}{\mu'} (F(x^{\itr{0}}) - \opt f)$, and $v^{\itr{0}} = x^{\itr{0}}$, then $\psi_0 \defeq \opt \psi_0 + \frac{\mu'}{2}\|x-v_0\|^2$ is a valid lower bound for $F$. In particular when $\lambda=LR^2$ then $F(x^{\itr{0}}) - \opt \psi_0 \le 2\kappa (F(x^{\itr{0}}) - \opt f)$.
\end{lem}

\begin{proof}
This lemma is a direct corollary of Lemma~\ref{lem:outerlowerbound} with $x^+ = x^{\itr{0}}$.
\end{proof}

\section{Dual APPA}
\label{sec:main_dual_ppa}

In this section we develop Dual APPA (Algorithm~\ref{alg:dual-appa}), a
natural approximate proximal point algorithm that operates
entirely in the regularized ERM dual.  Our focus here is on
theoretical properties of Dual APPA; Section~\ref{sec:implementation}
later explores aspects of Dual APPA more in practice.

We first present an abstraction for dual-based inner minimizers
(Section~\ref{sec:dual:oracles}), then present the algorithm
(Section~\ref{sec:dual:algo}), and finally step through its runtime
analysis (Section~\ref{sec:appendix-dual}).

\subsection{Approximate dual oracles}
\label{sec:dual:oracles}

Our primary goal in this section is to quantify how much objective function progress an algorithm needs to make in the dual problem, $g_{s,\lambda}$ (See Section~\ref{sec:setup}) in order to ensure primal progress at a rate  similar to
that in APPA (Algorithm~\ref{alg:primal-appa}).

Here, similar to Section~\ref{sub:appa:oracles}, we formally define our requirements for an approximate dual-based inner dual minimize. In particular, we use the following notion of dual oracle.

\begin{defn}
\label{defn:dual-oracle}
An algorithm $\oracle D$ is a $\emph{dual $(c,\lambda)$-oracle}$ if,
given $s \in \R^d$ and $y \in \R^n$, it outputs $\oracle D(s,y)$ that
is a $([g_{s,\lambda}(y) - \opt g_{s,\lambda}] / c)$-approximate
minimizer of $g_{s,\lambda}$ in time $\oracle T_{\oracle
D}$.\footnote{As in the primal oracle definition, when the oracle is
a randomized algorithm, we require that its output be an expected
$\ep$-approximate solution.}
\end{defn}

Dual based algorithms for regularized ERM and variants of coordinate descent typically can be used as such a dual oracle. In particular we note that APCG is such a
dual oracle.

\begin{thm}[APCG as a dual oracle]
\label{thm:apcg-dual}
APCG \citep{lin2014accelerated} is a dual $(c,\lambda)$-oracle with
runtime complexity $\oracle T_{\oracle D} = \tlO(nd \sqrt{\kappa_\lambda}
\log c)$.\footnote{As in Theorem~\ref{thm:accel_sdca_primal_oracle},
AP-SDCA could likely also serve as a dual oracle with the same
guarantees, provided it is modified to allow for the more general
primal-dual initialization.}
\end{thm}

\subsection{Algorithm}
\label{sec:dual:algo}

Our dual APPA is given by the following Algorithm~\ref{alg:dual-appa}.

\begin{algorithm}[H]
\begin{algorithmic}
\INPUT $x^\itr{0} \in \R^d$, $\lambda > 0$
\INPUT dual $(\sigma, \lambda)$-oracle $\oracle D$ \hfill (see Theorem~\ref{thm:unreg-dual-appa} for $\sigma$)
\STATE $y^\itr{0} \gets \hat y(x^\itr{0})$
\FOR{ $t = 1, \ldots, T$ }
\STATE $y^\itr{t} \gets \oracle D(x^\itr{t-1}, y^\itr{t-1})$
\STATE $x^\itr{t} \gets \hat x_{x^\itr{t-1},\lambda}(y^\itr{t})$
\ENDFOR
\OUTPUT $x^\itr{T}$
\end{algorithmic}
\caption{Dual APPA}
\label{alg:dual-appa}
\end{algorithm}

Dual APPA (Algorithm~\ref{alg:dual-appa}) repeatedly queries a dual oracle while producing primal
iterates via the dual-to-primal mapping \eqref{eq:d2p} along the way. We show that it obtains the following running time bound:

\begin{thm}[Un-regularizing in Dual APPA]
\label{thm:unreg-dual-appa}
Given a dual $\bp{ \sigma, \lambda }$-oracle
$\oracle D$, where
\begin{align*}
\sigma &\ge 80 n^2 \kappa_\lambda^2 \max\{\kappa,\kappa_\lambda\} \lceil \lambda / \mu \rceil
\end{align*}
Algorithm~\ref{alg:dual-appa} minimizes the ERM problem
\eqref{eq:erm} to within accuracy $\ep$ in time
$\tlO(
\oracle T_{\oracle D}
\lceil \lambda/\mu \rceil \log(\ep_0/\ep) )$.\footnote{As
in Theorem~\ref{thm:unreg-appa}, when the oracle is a
randomized algorithm, the expected accuracy is at most $\ep$.}
\end{thm}

Combining Theorem~\ref{thm:unreg-dual-appa} and Theorem~\ref{thm:apcg-dual} immediately yields another way to achieve our desired running time for solving \eqref{eq:erm}.

\begin{cor}
Instantiating Theorem~\ref{thm:unreg-dual-appa} with Theorem~\ref{thm:apcg-dual}  as the dual oracle and taking
$\lambda=\mu$ yields the running time bound
$\tlO(nd\sqrt{\kappa}\log(\ep_0/\ep))$.
\end{cor}

While both this result and the results in Section~\ref{sec:main_ppa}
show that APCG can be used to achieve our fastest running times for
solving \eqref{eq:erm}, note that the algorithms they suggest are in
fact different.
In every invocation of APCG in Algorithm~\ref{alg:primal-appa}, we
need to explicitly compute both the primal-to-dual and dual-to-primal
mappings (in $O(nd)$ time). However, here we only need to compute the
primal-to-dual mapping once upfront, in order to initialize the
algorithm.  Every subsequent invocation of APCG then only requires a
single dual-to-primal mapping computation, which can often be
streamlined.
From a practical viewpoint, this can be seen as a natural ``warm
start'' scheme for the dual-based inner minimizer.

\subsection{Analysis}
\label{sec:appendix-dual}

Here we proves Theorem~\ref{thm:unreg-dual-appa}.
We begin by bounding the error of the dual regularized ERM problem
when the center of regularization changes. This characterizes the
initial error at the beginning of each Dual APPA iteration.

\begin{lem}[Dual error after re-centering.]
\label{lem:dual-error-across-shift-old}
For all $y \in \R^n$, $x \in \R^d$, and $x' = \hat x_{x}(y)$ we have
\[
g_{x',\lambda}(y) - \opt g_{x',\lambda}
\leq
2 (g_{x,\lambda}(y) - \opt g_{x,\lambda})
+ 4 n \kappa
\left[ F(x') - \opt F + F(x) - \opt F
\right]
\]
\end{lem}
In other words, the dual error $g_{s,\lambda}(y) - \opt g_{s,\lambda}$ is bounded across a
re-centering step by multiples of previous sub-optimality
measurements (namely, dual error and gradient norm).
\begin{proof}
By the definition of $g_{x,\lambda}$ and $x'$ we have, for all $z$,
\[
g_{x',\lambda}(z) = G(z) + \frac 1 {2\lambda} \norm{A^\T z}^2 - x'^\T A^\T z
= g_{x,\lambda}(z) - (x' - x)^\top A^\T z
= g_{x,\lambda}(z) + \frac{1}{\lambda } y^\T A A^\T z
\enspace.
\]
Furthermore, since $g$ is $\frac{1}{L}$-strongly convex we can invoke Lemma~\ref{lem:addlinear} obtaining
\[
g_{x',\lambda}(y) - \opt g_{x',\lambda}
\leq 2 \left[g_{x,\lambda}(y) - \opt g_{x',\lambda}\right]
+ L \left\| \frac{1}{\lambda} A A^\T y \right\|_2^2 .
\]
Since each row of $A$ has $\ell_2$ norm at most $R$ we know that $\norm{A z}_2^2 \leq n R^2 \norm{z}_2^2$ and we know that by definition $A^\T y = \lambda(x - x')$. Combining these yields
\[
g_{x',\lambda}(y) - \opt g_{x',\lambda}
\leq 2 \left[g_{x,\lambda}(y) - \opt g_{x',\lambda}\right]
+ n L R^2 \norm{x - x'}_2^2 .
\]
Finally, since $F$ is $\mu$-strongly convex, by
Lemma~\ref{lem:smooth-sc-bounds}, we have
\[
\frac{1}{2} \norm{x - x'}_2^2
\leq \norm{x' - \opt x}_2^2 + \norm{x - \opt x}_2^2
\leq \frac{2}{\mu} \left[F(x') - \opt F + F(x) - \opt F\right]\enspace.
\]
Combining and recalling the definition of $\kappa$ yields the result.
\end{proof}

The following lemma establishes the rate of convergence of the primal
iterates $\set{x^\itr{t}}$ produced over the course of Dual APPA, and
in turn implies Theorem~\ref{thm:unreg-dual-appa}.

\begin{lem}[Convergence rate of Dual APPA]
\label{prop:dual-appa-converge-gradnorm} Let $c' \in (0, 1)$ be
arbitrary and suppose that $\sigma \geq ( 40 / c') n^2
\kappa_\lambda^2 \max\{\kappa,\kappa_\lambda\} \lceil \lambda / \mu
\rceil$ in Dual APPA (Algorithm~\ref{alg:dual-appa}).
Then in every iteration $t \geq 1$ of Dual APPA
(Algorithm~\ref{alg:dual-appa}) the following invariants hold:
\begin{align}
F(x^\itr{t-1}) - \opt F &\le \left(\frac{\lambda + c' \mu}{\lambda + \mu}\right)^{t - 1} \left(F(x^\itr{0}) - \opt{F}\right)
,
~~~ \textand \label{eq:dual-converge-invariant1} \\[1em]
g_{x^\itr{t-1},\lambda}(y^\itr{t}) - \opt g_{x^\itr{t-1},\lambda} &\le \left(\frac{\lambda + c' \mu}{\lambda + \mu}\right)^{t - 1} \left(F(x^\itr{0}) - \opt{F}\right).
\label{eq:dual-converge-invariant2}
\end{align}
\end{lem}
\newcommand\dualBoundsPrimalFactor{2 n \kappa_\lambda^2}
\begin{proof}
For notational convenience we let $r \defeq (\frac{\lambda + c'
\mu}{\lambda + \mu})$, $g_t \defeq g_{x^\itr{t}, \lambda}$, $f_t
\defeq f_{x^\itr{t},\lambda}$, and $\epsilon_t \defeq F(x^\itr{t}) -
\opt{F}$ for all $t \geq 0$. Thus, we wish to show that $\epsilon_{t
- 1} \leq r^{t-1} \epsilon_0$ (equivalent to
\eqref{eq:dual-converge-invariant2}) and we wish to show that $g_{t
- 1}(y^\itr{t}) - \opt{g}_{t - 1} \leq r^{t - 1} \epsilon_0$
(equivalent to \eqref{eq:dual-converge-invariant1}) for all $t \geq
1$.

By definition of a dual oracle we have, for all $t \geq 1$,
\begin{equation}
\label{eq:dual:progress}
g_{t- 1} (y^\itr{t})
- \opt{g}_{t - 1}
\leq \frac{1}{\sigma} \left[
g_{t - 1} (y_{t - 1})
- \opt{g}_{t - 1}
\right] ,
\end{equation}
by Lemma~\ref{lem:dual-error-bounds-primal} we have, for all $t \geq 1$,
\begin{equation}
\label{eq:dual:from-primal}
f_{t- 1}(x^\itr{t}) - \opt{f}_{t - 1}
\leq 2 n^2 \kappa_\lambda^2
\left[g_{t - 1}(y^\itr{t}) - \opt{g}_{t - 1} \right] ,
\end{equation}
by Lemma~\ref{lem:dual-error-across-shift-old} we know
\begin{equation}
\label{eq:dual:after-shift}
g_{t}(y^\itr{t}) - \opt{g}_{t}
\leq 2 \left[g_{t - 1}(y^\itr{t})  - \opt{g}_{t}\right]
+ 4 n \kappa (\epsilon_t + \epsilon_{t - 1}) ,
\end{equation}
and by Lemma~\ref{lem:proximal_point_progress} we know that for all $t \geq 1$
\begin{equation}
\label{eq:dual:proximal_point_progress}
\opt f_{t-1} - \opt{F}
\leq
\frac{\lambda}{\mu + \lambda} \epsilon_{t - 1}
\end{equation}
Furthermore, by Corollary~\ref{cor:init_dual_error}, the definition of $y^\itr{0}$, and the facts that $f_0(x^\itr{0}) = F(x^\itr{0})$ and $f_t(z) \geq F(z)$ we have
\begin{equation}
\label{eq:dual:initial}
g_{0}(y^\itr{0}) -
\opt{g}_{0} \leq
2 \kappa_\lambda \left(f_{0}(x^\itr{0}) -
\opt{f}_{0}  \right)
\leq
2 \kappa_\lambda \left(F(x^\itr{0}) - \opt{F} \right)
= 2 \kappa_\lambda \epsilon_0
\end{equation}
We show that combining these and applying strong induction on $t$ yields the desired result.

We begin with our base cases. When $t = 1$ the invariant \eqref{eq:dual-converge-invariant2} holds immediately by definition. Furthermore, when $t = 1$ we see that the invariant \eqref{eq:dual-converge-invariant1} holds, since $\sigma \geq 2 \kappa_\lambda$ and
\begin{align}
g_{0}(y^\itr{1}) - \opt g_{0}
&\le \frac 1 \sigma (g_{0}(y^\itr{0}) - \opt g_{0})
\le \frac {2 \kappa_\lambda} \sigma \bp{ f_{0}(x^\itr{0}) - \opt{f}_{0} }
\le \frac {2 \kappa_\lambda} \sigma \epsilon_0
,
\label{eq:base-case-dual-error}
\end{align}
were we used \eqref{eq:dual:progress} and \eqref{eq:dual:initial}  respectively. Finally we show that invariant \eqref{eq:dual-converge-invariant2} holds for $t = 2$:
\begin{align*}
F(x^\itr{1}) - \opt F
&\le f_0(x^\itr{1}) - \opt f_0 + \opt f_0 - \opt F
\tag{Since $F(z) \leq f_t(z)$ for all $t,z$}
\\
&\le 2 n^2 \kappa_\lambda^2 ( g_{0}(y^\itr{1}) - \opt g_{0} ) + \frac \lambda {\mu + \lambda} \epsilon_0
\tag{Equations \eqref{eq:dual:from-primal} and \eqref{eq:dual:proximal_point_progress}}
\\
&\le \bp{ \frac {4 n^2 \kappa_\lambda^3} \sigma + \frac \lambda {\mu + \lambda} } \epsilon_0
\tag{Equation~\eqref{eq:base-case-dual-error}}
\\
&\leq r \epsilon_0
\tag{Since $\sigma \geq 4n \kappa_\lambda^3 / (c' \lambda / (\mu + \lambda))$}
\end{align*}

Now consider $t \ge 3$ for the second invariant \eqref{eq:dual-converge-invariant2}. We show this holds assuming the invariants hold for all smaller $t$.
\begin{align*}
F(x^\itr{t-1}) - \opt F
&\le f_{t-2}(x^\itr{t-1}) - \opt f_{t-2} + \opt f_{t-2} - \opt F
\tag{Since $F(z) \leq f_t(z)$ for all $t,z$}
\\
&\le 2 n^2 \kappa_\lambda^2 ( g_{t - 2}(y_{t - 1}) - \opt g_{t - 2} ) + \frac \lambda {\mu + \lambda} \epsilon_{t - 2}
\tag{Equations \eqref{eq:dual:from-primal} and \eqref{eq:dual:proximal_point_progress}}
\\
&\le \frac {2 n^2 \kappa_\lambda^2 }{\sigma} \left( g_{t-2}(y_{t-2}) - \opt g_{t-2} \right)
+ \frac \lambda {\mu + \lambda} \epsilon_{t - 2}
\tag{Equation~\eqref{eq:dual:progress}}
\end{align*}
Furthermore,
\begin{align*}
g_{t - 2}(y_{t-2}) - \opt g_{t-2}
&\leq
2 (g_{t - 3}(y_{t-2}) - \opt g_{t - 3}) +
4 n \kappa \bb{ \epsilon_{t- 2} + \epsilon_{t - 3} }
\tag{Equation~\eqref{eq:base-case-dual-error}}
\\
&\leq \left(2r^{t - 2} + 4n \kappa (r^{t - 1} + r^{t -2})\right)
\epsilon_0
\tag{Inductive hypothesis}
\\
&\leq 10 n \kappa r^{t - 1} \tag{$r \leq 1$ and $\kappa \geq 1$} \epsilon_0
\end{align*}
Since $\sigma \geq 20n^2 \kappa_\lambda^2 \kappa / (c' \lambda / (\mu + \lambda))$ combining yields that
\[
\frac {2 n^2 \kappa_\lambda^2 }{\sigma} \left( g_{t-2}(y_{t-2}) - \opt g_{t-2} \right)
\leq
\frac{c' \mu}{\mu + \lambda}
r^{t - 1} \epsilon_0
\]
and the result follows by the inductive hypothesis on $\epsilon_{t - 2}$.

Finally we show that invariant \eqref{eq:dual-converge-invariant1}
holds for any $t \geq 2$ given that it holds for all smaller $t$ and
invariant \eqref{eq:dual-converge-invariant2} holds for that $t$ and
all smaller $t$.
\begin{align*}
g_{t-1}(y^\itr{t}) - \opt g_{t-1}
&\le \frac 1 \sigma (g_{t-1}(y_{t-1}) - \opt g_{t-1})
\tag{Definition dual oracle.}\\
&\le \frac{1}{\sigma} \bb{
2 (g_{x^\itr{t-2}}(y_{t-1}) - \opt g_{x^\itr{t-2}}) +
4 n \kappa \bb{ \epsilon_{t - 1} + \epsilon_{t - 2} }}
\tag{Equation~\eqref{eq:dual:after-shift}}\\
&\le \frac 1 \sigma \bb{
2 r^{t-1} +
4n\kappa \bb{ r^t + r^{t-1} } } \epsilon_0
\tag{Inductive hypothesis} \\
&\le r^{t - 1} \epsilon_0
\tag{$\sigma \geq 8n \kappa$}
\end{align*}
The result then follows by induction.
\end{proof}

\section{Implementation}
\label{sec:implementation}

In the following two subsections, respectively, we discuss
implementation details and report on an empirical evaluation of the
APPA framework.

\subsection{Practical concerns}
\label{sec:practical}

While theoretical convergence rates lay out a broad-view comparison of
the algorithms in the literature, we briefly remark on some of the
finer-grained differences between algorithms, which inform their
implementation or empirical behavior. To match the terminology used
for SVRG in \citet{johnson13svrg}, we refer to a ``stage'' as a single
step of APPA, \ie the time spent executing the inner minimization of
$f_{x^\itr{t},\lambda}$ or $g_{x^\itr{t},\lambda}$ (as in
\eqref{eq:reg-primal} and \eqref{eq:reg-dual}).

\paragraph{Re-centering overhead of \dualappa{} vs.\ SVRG}
At the end of every one of its stages, SVRG pauses to compute an exact
gradient by a complete pass over the dataset (costing $\Theta(nd)$
time during which $n$ gradients are computed). Although an amortized
runtime analysis hides this cost, this operation cannot be carried out
in-step with the iterative updates of the previous stage, since the
exact gradient is computed at a point that is only selected at the
stage's end.

Meanwhile, if each stage in \dualappa{} is initialized with a valid
primal-dual pair for the inner problem, \dualappa{} can update the
current primal point together with every dual coordinate update, in
time $O(d)$, \ie with negligible increase in the overhead of the
update. When doing so, the corresponding data row remains fresh in
cache and, unlike SVRG, no additional gradient need be computed.

Moreover, initializing each stage with a valid such primal-dual pair
can be done in only $O(d)$ time. At the end of a stage where $s$ was
the center point, \dualappa{} holds a primal-dual pair $(x, y)$ where
$x = \hat x_s(y)$. The next stage is centered at $x$ and the dual
variables initialized at $y$, so it remains to set up a corresponding
primal point $x' = \hat x_x(y) = x - \frac 1 \lambda A^\T y$. This can
be done by computing
$x' \gets 2x - s$,
since we know that $x-s=-\frac 1 \lambda A^\T y$.

\paragraph{Decreasing $\lambda$} APPA and Dual APPA enjoy the nice property
that, as long as the inner problems are solved with enough accuracy,
the algorithm does not diverge even for large choice of $\lambda$. In
practice this allows us to start with a large $\lambda$ and make
faster inner minimizations. If we heuristically observe that the
function error is not decreasing rapidly enough, we can switch to a
smaller $\lambda$. Figure~\ref{fig:sensitivity}
(Section~\ref{sec:empirical}) demonstrates this empirically.
This contrasts with algorithm parameters such as step size choices in
stochastic optimizers (that may still appear in inner minimization).
Such parameters are typically more sensitive, and can suddenly lead to
divergence when taken too large, making them less amenable to
mid-run parameter tuning.

\paragraph{Stable update steps} When used as inner minimizers, dual
coordinate-wise methods such as SDCA typically provide a convenient
framework in which to derive parameter updates with data-dependent
step sizes, or sometimes enables closed-form updates altogether (\ie
optimal solutions to each single-coordinate maximization
sub-problem). For example, when Dual APPA is used together with SDCA
to solve a problem of least-squares or ridge regression, the locally
optimal SDCA updates can be performed efficiently in closed form.
This decreases the number of algorithmic parameters requiring tuning,
improves the overall the stability of the end-to-end optimizer and,
in turn, makes it easier to use out of the box.

\subsection{Empirical analysis}
\label{sec:empirical}

We experiment with \dualappa{} in comparison with SDCA, SVRG, and SGD
on several binary classification tasks.

Beyond general benchmarking, the experiments also demonstrate the
advantages of the unordinary ``bias-variance tradeoff'' presented by
approximate proximal iteration: the vanishing proximal term
empirically provides advantages of regularization (added strong
convexity, lower variance) at a bias cost that is less severe than
with typical $\ell_2$ regularization. Even if some amount of $\ell_2$
shrinkage is desired, \dualappa{} can place yet higher weight on its
$\ell_2$ term, enjoy improved speed and stability, and after a few
stages achieve roughly the desired bias.

\paragraph{Datasets}
In this section we show results for three binary
classification tasks, derived from
MNIST,\footnote{\url{http://yann.lecun.com/exdb/mnist/}}
CIFAR-10,\footnote{\url{http://www.cs.toronto.edu/~kriz/cifar.html}}
and
Protein:\footnote{\url{http://osmot.cs.cornell.edu/kddcup/datasets.html}}
in \mnist{} we classify the digits $\set{1, 2, 4, 5, 7}$ vs.\ the
rest, and in \cifar{} we classify the animal categories vs.\ the
automotive ones.
\mnist{} and \cifar{} are taken under non-linear feature transformations that
increase the problem scale significantly: we normalize the rows by
scaling the data matrix by the inverse average $\ell_2$ row norm.
We then take take $n/5$ random Fourier features per the randomized
scheme of \citet{rahimi2007random}. This yields 12K features for
\mnist{} (60K training examples, 10K test) and 10K for \cifar{} (50K
training examples, 10K test). Meanwhile, \protein{} is a standard
pre-featurized benchmark (75 features, $\sim$117K training examples,
$\sim$30K test) that we preprocess minimally by row normalization and
an appended affine feature, and whose train/test split we obtain by
randomly holding out 20\% of the original labeled data.

\newcommand\experimentssize{.85}

\begin{figure}[ht!]
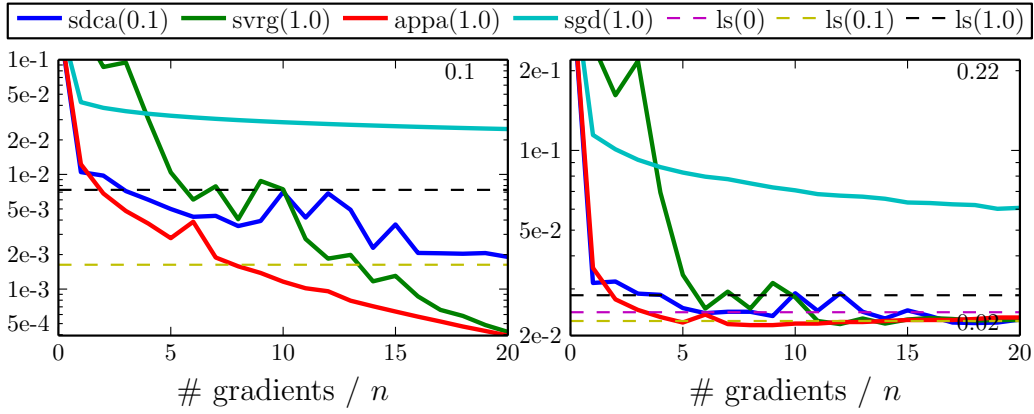
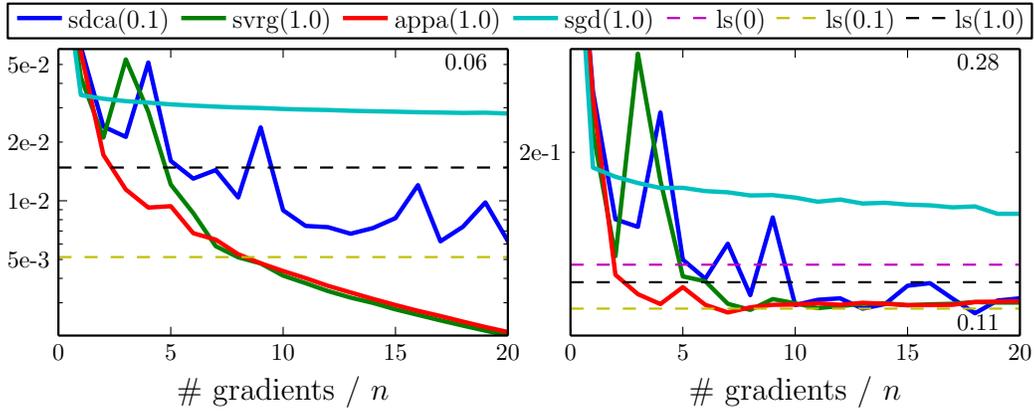
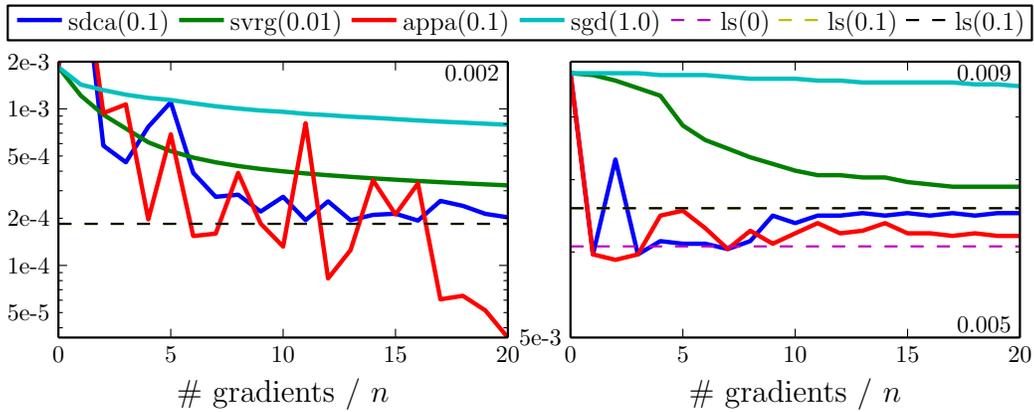

\centering
\subfigure[\mnist{}. Left: excess train loss $F(x) - \opt F$. Right: test error rate.]{
\includePgfSized[\experimentssize\textwidth]{mnist-train-accuracy-squared.pgf}}

\subfigure[\cifar{}. Left: excess train loss $F(x) - \opt F$. Right: test error rate.]{
\includePgfSized[\experimentssize\textwidth]{cifar-train-accuracy-squared.pgf}}

\subfigure[\protein{}. Left: excess train loss $F(x) - \opt F$. Right: test error rate.]{
\includePgfSized[\experimentssize\textwidth]{protein-train-accuracy-squared.pgf}}
\caption{Sub-optimality curves when optimizing under squared loss $\phi_i(z) = \tfrac 1 {2n} (z - b_i)^2$.}
\label{fig:squared-learning-curves}
\end{figure}

\begin{figure}[ht!]
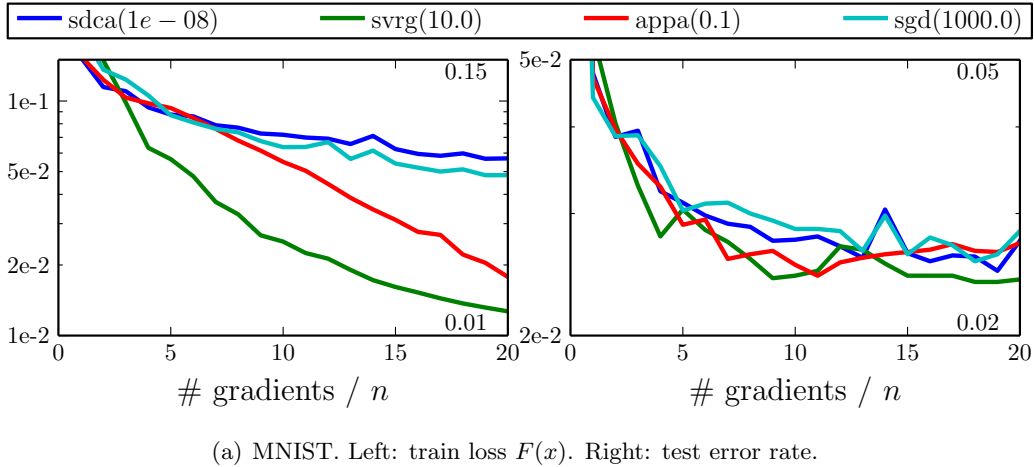
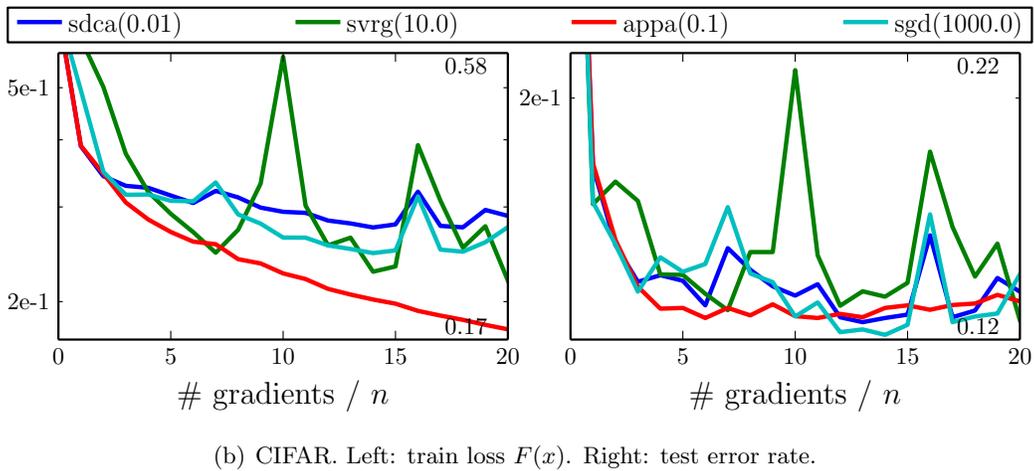

\centering
\subfigure[\mnist{}. Left: train loss $F(x)$. Right: test error rate.]{
\includePgfSized[\experimentssize\textwidth]{mnist-train-accuracy-logistic.pgf}}

\subfigure[\cifar{}. Left: train loss $F(x)$. Right: test error rate.]{
\includePgfSized[\experimentssize\textwidth]{cifar-train-accuracy-logistic.pgf}}
\caption{Objective curves when optimizing under logistic loss $\phi_i(z) = \tfrac 1 n \log(1+e^{-zb_i})$.}
\label{fig:logistic-learning-curves}
\end{figure}

\begin{figure}[ht!]
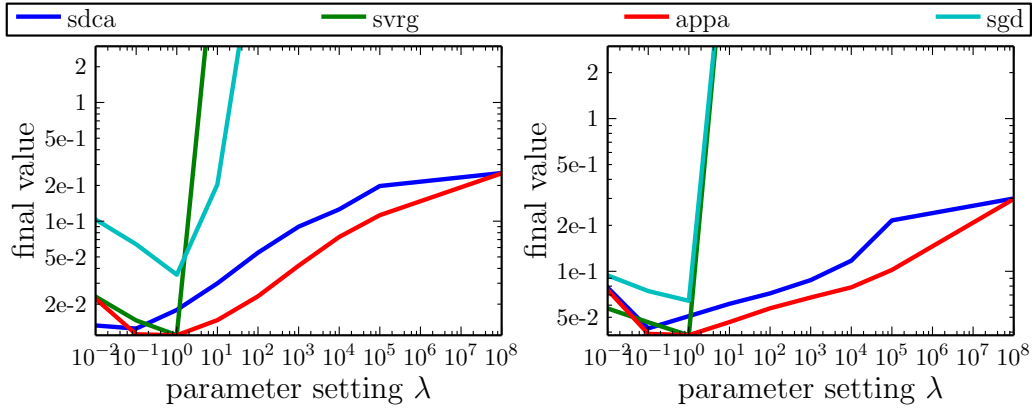
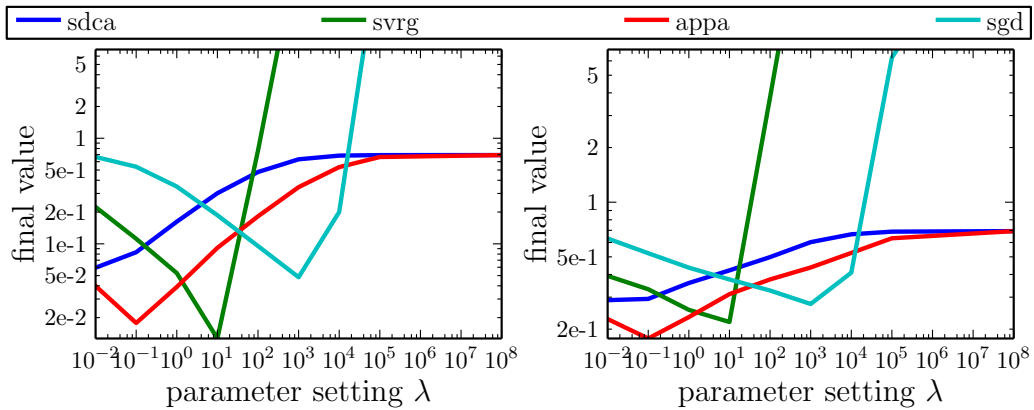

\centering
\subfigure[Squared loss. Left: \mnist{}. Right: \cifar{}.]{
\includePgfSized[\experimentssize\textwidth]{sens-cifar-mnist-squared.pgf}}

\subfigure[Logistic loss. Left: \mnist{}. Right: \cifar{}.]{
\includePgfSized[\experimentssize\textwidth]{sens-cifar-mnist-logistic.pgf}}
\caption{Sensitivity to $\lambda$: the final objective values attained
by each algorithm, after 20 stages (or the equivalent), with
$\lambda$ chosen at different orders of magnitude. SGD and SVRG
exhibit a sharp threshold past which they easily diverge, whereas
SDCA degrades more gracefully, and \dualappa{} yet more so.}
\label{fig:sensitivity}
\end{figure}

\paragraph{Algorithms}
Each algorithm is parameterized by a scalar value $\lambda$ analogous
to the $\lambda$ used in proximal iteration: $\lambda$ is the step
size for SVRG, $\lambda t^{-1/2}$ is the decaying step size for SGD,
and $\frac \lambda 2 \norm{x}_2^2$ is the ridge penalty for SDCA. (See
\citet{johnson13svrg} for a comparison of SVRG to a more thoroughly
tuned SGD under different decay schemes.) We use \dualappa{}
(Algorithm~\ref{alg:dual-appa}) with SDCA as the inner minimizer. For
the algorithms with a notion of a stage~-- \ie Dual APPA's time spent
invoking the inner minimizer, SVRG's period between computing exact
gradients~-- we set the stage size equal to the dataset size for
simplicity.\footnote{Such a choice is justified by the observation
that doubling the stage size does not have noticeable effect on the
results discussed.}
SVRG is given an advantage in that we choose not to count its gradient
computations when it computes the exact gradient between stages.
All algorithms are initialized at $x=0$.
Each algorithm was run under $\lambda=10^i$ for $i=-8, -7, \dots, 8$,
and plots report the trial that best minimized the original ERM
objective.

\paragraph{Convergence and bias}
The proximal term in APPA introduces a vanishing bias for the problem
(towards the initial point of $x=0$) that provides a speedup by adding
strong convexity to the problem. We investigate a natural baseline:
for the purpose of minimizing the original ERM problem, how does APPA
compare to solving one instance of a regularized ERM problem (using a
single run of its inner optimizer)? In other words, to what extent
does re-centering the regularizer over time help in solving the
un-regularized problem?  Intuitively, even if SDCA is run to
convergence, some of the minimization is of the regularization term
rather than the ERM term, hence one cannot weigh the regularization
too heavily. Meanwhile, APPA can enjoy more ample strong convexity by
placing a larger weight on its $\ell_2$ term. This advantage is
evident for \mnist{} and \cifar{} in
Figures~\ref{fig:squared-learning-curves}
and~\ref{fig:logistic-learning-curves}: recalling that $\lambda$ is
the same strong convexity added both by APPA and by SDCA, we see that
APPA takes $\lambda$ at least an order of magnitude larger than SDCA
does, to achieve faster and more stable convergence towards an
ultimately lower final value.

Figure~\ref{fig:squared-learning-curves} also shows dashed lines
corresponding to the ERM performance of the least-squares fit and of
fully-optimized ridge regression, using $\lambda$ as that of the best
APPA and SDCA runs.
These appear in the legend as ``ls($\lambda$).'' They indicate lower
bounds on the ERM value attainable by \emph{any} algorithm that
minimizes the corresponding regularized ERM objective.
Lastly, test set classification accuracy demonstrates the extent to
which a shrinkage bias is statistically desirable. In the \mnist{} and
\cifar{} holdout, we want only the small bias taken explicitly by SDCA
(and effectively achieved by APPA). In the \protein{} holdout, we want
no bias at all (again effectively achieved by APPA).

\paragraph{Parameter sensitivity} By solving only regularized ERM
inner problems, SDCA and APPA enjoy a stable response to poor
specification of the biasing parameter
$\lambda$. Figure~\ref{fig:sensitivity} plots the algorithms' final
value after 20 stages, against different choices of
$\lambda$. Overestimating the step size in SGD or SVRG incurs a sharp
transition into a regime of divergence. Meanwhile, APPA and SDCA
always converge, with solution quality degrading more smoothly. APPA
then exhibits an even better degradation as it overcomes an
overaggressive biasing by the 20th stage.

\section*{Acknowledgments}

Part of this work took place while RF and AS were at Microsoft
Research, New England, and another part while AS was visiting the
Simons Institute for the Theory of Computing, UC Berkeley. This work
was partially supported by NSF awards 0843915 and 1111109, NSF
Graduate Research Fellowship (grant no. 1122374).

\bibliography{all}
\bibliographystyle{plainnat}

\appendix

\section{Technical lemmas}
\label{sec:appendix-lemmas}

In this section we provide several stand-alone technical lemmas we use throughout the paper. First we provide Lemma~\ref{lem:smooth-sc-bounds} some common inequalities regarding smooth or strongly convex functions, then Lemma~\ref{lem:addlinear} which shows the effect of adding a linear term to a convex function, and then Lemma~\ref{lem:mergequadratic} a small technical lemma regarding convex combinations of quadratic functions.

\begin{lem}[Standard bounds for smooth, strongly convex functions]
\label{lem:smooth-sc-bounds}
Let $f : \R^k \rightarrow \R$ be differentiable function that obtains its minimal value at $\opt{x}$.

If $f$ is $L$-smooth then for all $x \in
\R^k$
\begin{align*}
\frac{1}{2L} \| \nabla f(x) \|_2^2
\le f(x) - f(\opt{x})
\le \frac{L}{2} \|x - \opt{x}\|_2^2 \enspace.
\end{align*}

If $f$ is $\mu$-strongly convex the for all $x \in
\R^k$
\begin{align*}
\frac{\mu}{2} \|x - \opt{x}\|_2^2
\leq f(x) - f(\opt{x})
\leq \frac{1}{2\mu} \| \nabla f(x) \|_2^2 \enspace.
\end{align*}
\end{lem}

\begin{proof}
Apply the definition of smoothness and strong convexity at the
points $x$ and $\opt x$ and minimize the resulting quadratic form.
\end{proof}

\begin{lem}
\label{lem:addlinear}
Let $f : \R^n \rightarrow \R$ be a $\mu$-strongly convex function and for all $a,x \in \R^n$ let $f_a(x) = f(x) + a^\top x$. Then
\[
f_a(x) - \opt f_a \leq  2 (f(x) - \opt{f}) + \frac{1}{\mu} \norm{a}_2^2
\]
\end{lem}

\begin{proof}\footnote{Note we could have also proved this by appealing to the gradient of $f$ and Lemma~\ref{lem:smooth-sc-bounds}, however the proof here holds even if $f$ is not differentiable.}
Let $\opt x = \argmin_x f(x)$. Since $f$ is $\mu$-strongly convex by Lemma~\ref{lem:smooth-sc-bounds} we have $f(x) \geq f(\opt x) + \frac{\mu}{2} \norm{x - \opt x}_2^2$ for all $x$. Consequently, for all $x$
\[
\opt f_a
\geq f(x) + a^\top x
\geq f(\opt x) + \frac{\mu}{2} \norm{x - \opt x}_2^2 + a^\top x
\geq f_a(\opt x) + a^\top (x - \opt x) + \frac{\mu}{2} \norm{x - \opt x}_2^2
\]
Minimizing with respect to $x$ yields that
$
\opt f_{a} \geq f_a(\opt x) - \frac{1}{2\mu} \norm{a}_2^2
$.
Consequently, by Cauchy Schwarz, and Young's Inequality we have
\begin{align}
f_a(x) - \opt f_a
&\leq f(x) - \opt{f} + a^\top(x - \opt x) + \frac{1}{2\mu} \norm{a}_2^2
\\
&\leq f(x) - \opt f + \frac{1}{2 \mu} \norm{a}_2^2 + \frac{\mu}{2} \norm{x - \opt{x}}_2^2
+\frac{1}{2\mu} \norm{a}_2^2
\end{align}
Applying \ref{lem:smooth-sc-bounds}  again yields the result.
\end{proof}

\begin{lem}
\label{lem:mergequadratic}
Suppose that for all $x$ we have $$f_{1}(x)\defeq\psi_{1}+\frac{\mu}{2}\norm{x-v_{1}}_{2}^{2}\mbox{ and }f_{2}(x)=\psi_{2}+\frac{\mu}{2}\norm{x-v_{2}}_{2}^{2}$$
then $$\alpha f_{1}(x)+(1-\alpha)f_{2}(x)=\psi_{\alpha}+\frac{\mu}{2}\norm{x-v_{\alpha}}_{2}^{2}$$
where
$$v_{\alpha}=\alpha v_{1}+(1-\alpha)v_{2}
\enspace \text{ and } \enspace
\psi_{\alpha}=\alpha\psi_{1}+(1-\alpha)\psi_{2}+\frac{\mu}{2}\alpha(1-\alpha)\norm{v_{1}-v_{2}}_{2}^{2}$$
\end{lem}

\begin{proof}
Setting the gradient of $\alpha f_{1}(x)+(1-\alpha)f_{2}(x)$
to $0$
we know that $v_{\alpha}$
must satisfy
$$\alpha\mu\left(v_{\alpha}-v_{1}\right)+(1-\alpha)\mu\left(v_{\alpha}-v_{2}\right)=0$$
and thus $v_{\alpha}=\alpha v_{1}+(1-\alpha)v_{2}$.
Finally,
\begin{align*}
\psi_{\alpha}  &=\alpha\left[\psi_{1}+\frac{\mu}{2}\norm{v_{\alpha}-v_{1}}_{2}^{2}\right]+\left(1-\alpha\right)\left[\psi_{2}+\frac{\mu}{2}\norm{v_{\alpha}-v_{2}}_{2}^{2}\right]\\
&=\alpha\psi_{1}+(1-\alpha)\psi_{2}+\frac{\mu}{2}\left[\alpha(1-\alpha)^{2}\norm{v_{2}-v_{1}}_{2}^{2}+(1-\alpha)\alpha^{2}\norm{v_{2}-v_{1}}_{2}^{2}\right]\\
&=\alpha\psi_{1}+(1-\alpha)\psi_{2}+\frac{\mu}{2}\alpha(1-\alpha)\norm{v_{1}-v_{2}}_{2}^{2}.
\end{align*}
\end{proof}

\section{Regularized ERM duality}
\label{sec:appendix-problems}

In this section we derive the dual \eqref{eq:reg-dual} to the problem of computing proximal operator for the ERM objective \eqref{eq:reg-primal} (Section~\ref{sub:app:dual:derive}) and prove several bounds on primal and dual errors (Section~\ref{sub:app:dual:error}). Throughout this section we assume $F$ is given by the ERM problem \eqref{eq:erm} and we make extensive use of the notation and assumptions in Section~\ref{sec:setup}.

\subsection{Dual derivation}
\label{sub:app:dual:derive}

We can rewrite the primal problem, $\min_x f_{s,\lambda}(x)$, as
\begin{align*}
\minproblem
{x \in \R^d, z \in \R^n}
{\nsum{i} \phi_i(z_i) + \frac \lambda 2 \norm{x-s}_2^2}
{& z_i = a_i^\T x, \quad \text{for } i = 1, \ldots, n }.
\end{align*}
By convex duality, this is equivalent to
\begin{align*}
\min_{x,\set{z_i}} \max_{y \in \R^n}
\nsum{i} \phi_i(z_i) + \frac \lambda 2 \norm{x-s}_2^2 + y^\T (A x - z)
&=
\max_{y} \min_{x,\set{z_i}}
\nsum{i} \phi_i(z_i) + \frac \lambda 2 \norm{x-s}_2^2 + y^\T (A x - z)
\end{align*}
Since
\[
\min_{z_i} \left\{\phi_i(z_i) - y_i z_i\right\}
= - \max_{z_i} \left\{  y_i z_i - \phi_i(z_i)\right\}
= - \phi_i^*(y_i)
\]
and
\[
\min_{x} \left\{\frac{\lambda}{2} \norm{x - s}_2^2 + y^\T A x \right\}
= y^\T A s + \min_{x} \left\{\frac{\lambda}{2} \norm{x - s}_2^2 + y^\T A (x - s)\right\}
= y^\T A s - \frac{1}{2\lambda }\norm{A^\T y}_2^2,
\]
it follows that the optimization problem is in turn equivalent to
\[
- \min_{y} \nsum{i} \phi_i^*(y_i) +
\frac 1 {2\lambda} \norm{A^\T y}_2^2 - s^\T A^\T y .
\]
This negated problem is precisely the dual formulation.

The first problem is a Lagrangian saddle-point problem, where the
Lagrangian is defined as
\begin{align*}
\mL(x,y,z) &= \nsum{i} \phi_i(z_i) + \frac \lambda 2 \norm{x-s}_2^2 + y^\T (A x - z).
\end{align*}
The dual-to-primal mapping \eqref{eq:d2p} and primal-to-dual mapping
\eqref{eq:p2d} are implied by the KKT conditions under $\mL$, and can
be derived by solving for $x$, $y$, and $z$ in the system $\grad \mL(x,y,z) = 0$.

The \emph{duality gap} in this context is defined as
\begin{align}
\gap_{s,\lambda}(x,y) &\defeq f_{s,\lambda}(x) + g_{s,\lambda}(y).
\label{eq:gap}
\end{align}
Strong duality dictates that $\gap_{s,\lambda}(x,y) \geq 0$ for all $x
\in \R^d$, $y \in \R^n$, with equality attained when $x$ is
primal-optimal and $y$ is dual-optimal.

\subsection{Error bounds}
\label{sub:app:dual:error}

\begin{lem}[Dual error bounds primal error]
\label{lem:dual-error-bounds-primal}
For all $s \in \R^d$, $y \in \R^n$, and $\lambda > 0$ we have
\begin{align*}
f_{s,\lambda}(\hat{x}_{s,\lambda}(y)) - \opt f_{s,\lambda}
&\leq 2 (n \kappa_\lambda)^2  (g_{s,\lambda}(y) - \opt g_{s,\lambda}) .
\end{align*}
\end{lem}

\begin{proof}
Because $F$ is $n R^2 L$ smooth, $f_{s,\lambda}$ is $n R^2 L +
\lambda$ smooth. Consequently, for all $x \in \R^d$ we have
\begin{align*}
f_{s,\lambda}(x) - \opt f_{s,\lambda}
&\leq \frac {n R^2 L + \lambda} 2 \norm{ x - \opt{x_{s,\lambda}} }_2^2
\end{align*}
Since we know that $\opt{x_{s,\lambda}} = s - \frac 1 \lambda A^\T
\opt{y_{s,\lambda}}$ and $\norm{A^\T z}_2^2 \leq n R^2 \norm{z}_2^2$ for all $z \in \R^n$ we have
\begin{align*}
f_{s,\lambda}(\hat{x}_{x,\lambda}(y)) -
f_{s,\lambda}(\opt{x_{s,\lambda}})
&\leq \frac {n R^2 L + \lambda} 2 \norm{ s - \frac 1 \lambda A^\T y - (s - \frac 1 \lambda A^\T \opt{y_{s,\lambda}}) }_2^2 \\
&=    \frac {n R^2 L + \lambda} {2 \lambda^2} \norm{ y - \opt{y_{s,\lambda}} }_{A A^\T}^2 \\
&\leq \frac {n R^2 (n R^2 L + \lambda)} {2 \lambda^2} \norm{ y - \opt{y_{s,\lambda}} }_2^2. \num
\label{eq:lem:dual-error-bounds-primal:1}
\end{align*}
Finally, since each $\phi_i^*$ is $1/L$-strongly convex, $G$ is
$1/L$-strongly convex and hence so is
$g_{s,\lambda}$. Therefore by Lemma~\ref{lem:smooth-sc-bounds} we have
\begin{align}
\frac{1}{2L} \norm{ y - \opt{y_{s,\lambda}} }_2^2
&\leq g_{s,\lambda}(y) - g_{s,\lambda}(\opt{y_{s,\lambda}}).
\label{eq:lem:dual-error-bounds-primal:2}
\end{align}
Substituting \eqref{eq:lem:dual-error-bounds-primal:2} in
\eqref{eq:lem:dual-error-bounds-primal:1} and recalling that $\kappa_\lambda \geq 1$ yields the result.
\end{proof}

\begin{lem}[Gap for primal-dual pairs]
\label{lem:gaps-grads}
For all $s,x \in \R^d$ and $\lambda > 0$ we have
\begin{align}
\gap_{s,\lambda}(x, \hat y(x))
&= \frac 1 {2\lambda} \norm{\grad F(x)}_2^2 +
\frac \lambda 2 \norm{x-s}_2^2.
\label{eq:gap-equals-primalgradnorm}
\end{align}

\end{lem}
\begin{proof}
To prove the first identity \eqref{eq:gap-equals-primalgradnorm},
let $\hat y = \hat y(x)$ for brevity. Recall that
\begin{align}
\hat y_i
&= \phi_i'(a_i^\T x) \in \argmax_{y_i} \set{x^\T a_i y_i - \phi_i^*(y_i) }
\label{eq:yhat-fenchel}
\end{align}
by definition, and hence $x^\T a_i \hat y_i - \phi_i^*(\hat y_i) =
\phi_i(a_i^\T x)$. Observe that
\begin{align*}
\gap_{s,\lambda}(x, \hat y)
&= \sum_{i=1}^n \bigpar{ \phi_i(a_i^\T x) + \phi_i^*(\hat y_i) } - x^\T A^\T \hat y +
\tfrac 1 {2\lambda} \norm{A^\T \hat y}^2 + \tfrac \lambda 2 \norm{x-s}^2 \\
&= \sum_{i=1}^n \bigpar{ \underbrace{\phi_i(a_i^\T x) + \phi_i^*(\hat y_i) - x^\T a_i \hat y_i}_{= 0 \text{ (by \eqref{eq:yhat-fenchel})}} } +
\tfrac 1 {2\lambda} \norm{A^\T \hat y}^2 + \tfrac \lambda 2 \norm{x-s}^2 \\
&= \tfrac 1 {2\lambda} \norm{A^\T \hat y}^2 + \tfrac \lambda 2 \norm{x-s}^2 \\
&= \tfrac 1 {2\lambda} \norm{\sum_{i=1}^n a_i \phi_i'(a_i^\T x)}^2 + \tfrac \lambda 2 \norm{x-s}^2 \\
&= \tfrac 1 {2\lambda} \norm{\grad F(x)}^2 + \tfrac \lambda 2 \norm{x-s}^2.
\end{align*}

\end{proof}

\begin{cor}[Initial dual error]
\label{cor:init_dual_error}
For all $s,x \in \R^d$ and $\lambda > 0$ we have
\begin{align*}
g_{x,\lambda}(\hat{y}(x)) - \opt{g}_{x,\lambda}
&\le
2 \kappa_{\lambda}
\bp{ f_{x,\lambda}(x) - \opt{f}_{x,\lambda} }
\end{align*}
\end{cor}

\begin{proof}
By Lemma~\ref{lem:gaps-grads} we have
\[
\gap_{x,\lambda}(x,\hat{y}(x))
= \frac{1}{2\lambda} \norm{\grad F(x)}_2^2 + \frac{\lambda}{2} \norm{x - x}_2^2
= \frac{1}{2\lambda} \norm{\grad F(x)}_2^2
\]
Now clearly $\grad F(x) = \grad f_{x,\lambda}(x)$. Furthermore, since $f_{x,\lambda}(x)$ is $(n L R^2 + \lambda)$-smooth by Lemma~\ref{lem:smooth-sc-bounds} we have $\norm{\grad f_{x, \lambda}(x)} \leq 2(n L R^2 + \lambda) (f_{x,\lambda}(x) - \opt{f}_{x,\lambda})$. Consequently,
\[
g_{x,\lambda}(\hat{y}(x)) - \opt{g}_{x,\lambda}
\leq \gap_{x,\lambda}(x,\hat{y}(x))
\leq \frac{2 (n L R^2 + \lambda)}{2 \lambda} \left(f_{x,\lambda}(x) - \opt{f}_{x,\lambda}\right)\enspace.
\]
Recalling the definition of $\kappa_\lambda$ and the fact that $1 \leq \kappa_\lambda$ yields the result.
\end{proof}

\clearpage

\end{document}